\documentclass[twoside, journal]{IEEEtran}

\usepackage{cite}
\usepackage{url}
\usepackage{multicol}
\usepackage[normalem]{ulem}

\usepackage{titlesec}

\newcommand{\ssubsection}[1]{%
 \subsection*{\textbf{\raggedright\normalfont#1}}}

\def\nicefrac#1#2{\leavevmode%
 \raise.5ex\hbox{\small #1}%
 \kern-.1em/\kern-.15em%
 \lower.25ex\hbox{\small #2}}
 
\let\temp\rmdefault 
\usepackage{mathpazo} 
\let\rmdefault\temp

\usepackage{tabularx,arydshln}

\usepackage{siunitx}
\protected\def\numpi{\text{\ensuremath{\pi}}}
\usepackage[pdfauthor= { Giorgos Mamakoukas; Maria L. Castano; Xiaobo Tan; Todd D. Murphey},
 pdftitle={Derivative-Based Koopman Operators for Real-Time Control of Robotic Systems},
 pdfsubject={Derivative-Based Koopman Operators for Real-Time Control of Robotic Systems},
 pdfkeywords={Koopman Operators, Data-Driven Control, Error Bounds, Robotic Fish},colorlinks = true,
 urlcolor = blue,]{hyperref}

\usepackage{balance}
\usepackage{siunitx} 
\usepackage{adjustbox} 

\usepackage{dcolumn}
\newcolumntype{L}{D{.}{.}{2,5}}

\usepackage{xcolor}
\usepackage{amsmath} 
\usepackage{amsfonts} 
\usepackage{subcaption} 
\usepackage{graphicx}
\usepackage{mwe}
\usepackage{tabularx} 

\usepackage{amsthm} 
\newtheorem{theorem}{Theorem}
\newtheorem{corollary}{Corollary}

\usepackage{amssymb}
\newlength{\tempdima}
\newcommand{\rowname}[1]
{\rotatebox{90}{\makebox[\tempdima][c]{\textbf{#1}}}}
\newcommand{\sgn}{\mathop{\mathrm{sgn}}} 
\usepackage{threeparttable} 

\usepackage{soul}

\usepackage{tikz}
\usepackage{orcidlink}

\ifCLASSINFOpdf
\else
\fi

\begin{document}
\title{Derivative-Based Koopman Operators for Real-Time Control of Robotic Systems}


\author{
    Giorgos~Mamakoukas$^{\orcidlink{0000-0002-3461-0849}}$, 
     Maria~L.~Casta\~{n}o$^{\orcidlink{0000-0001-5552-6817}}$, 
     Xiaobo~Tan$^{\orcidlink{0000-0002-5542-6266}}$,~and 
    Todd~D.~Murphey$^{\orcidlink{0000-0003-2262-8176}}$
\thanks{Manuscript received October 5, 2020; revised February  15, 2021. This article was recommended for publication by the Editor Paolo Robuffo Giordano upon evaluation of the reviewers’ comments. This work was supported by the National Science Foundation (IIS-1717951, IIS-1715714, DGE1424871). Any opinions, findings, and conclusions or recommendations expressed in this material are those of the authors and do not necessarily reflect the views of the National Science Foundation.}
\thanks{G. Mamakoukas and Todd D. Murphey are with the Department of Mechanical Engineering, Northwestern University, Evanston, Illinois 60208, USA (e-mail: \href{mailto:giorgosmamakoukas@u.northwestern.edu}{giorgosmamakoukas@u.northwestern.edu).}}
\thanks{Maria~L.~Casta\~{n}o and Xiaobo Tan are with the Department of Electrical and Computer Engineering, Michigan State University, East Lansing, Michigan 48824, USA.}
\thanks{This article has supplementary material and code at \href{https://github.com/giorgosmamakoukas/dataDrivenControlOfRoboticFish}{https://github.com/giorgosmamakoukas/dataDrivenControlOfRoboticFish}.}
}

%
%

\markboth{IEEE TRANSACTIONS ON ROBOTICS, IN PRESS}{Mamakoukas \MakeLowercase{\textit{et al.}}: Derivative-Based Koopman Operators for Real-Time Control of Robotic Systems}




\maketitle

\IEEEpubid{\begin{minipage}{\textwidth}\ \vspace{6ex}\\[8pt] 
  \copyright 2021 IEEE.  Personal use of this material is permitted.  Permission from IEEE must be obtained for all other uses, in any current or future media, including reprinting/republishing this material for advertising or promotional purposes, creating new collective works, for resale or redistribution to servers or lists, or reuse of any copyrighted component of this work in other works.
\end{minipage}} 

\begin{abstract}
This paper presents a generalizable methodology for data-driven identification of nonlinear dynamics that bounds the model error in terms of the prediction horizon and the magnitude of the derivatives of the system states. Using higher-order derivatives of general nonlinear dynamics that need not be known, we construct a Koopman operator-based linear representation and utilize Taylor series accuracy analysis to derive an error bound. The resulting error formula is used to choose the order of derivatives in the basis functions and obtain a data-driven Koopman model using a closed-form expression that can be computed in real time. Using the inverted pendulum system, we illustrate the robustness of the error bounds given noisy measurements of unknown dynamics, where the derivatives are estimated numerically. When combined with control, the Koopman representation of the nonlinear system has marginally better performance than competing nonlinear modeling methods, such as SINDy and NARX. In addition, as a linear model, the Koopman approach lends itself readily to efficient control design tools, such as LQR, whereas the other modeling approaches require nonlinear control methods. The efficacy of the approach is further demonstrated with simulation and experimental results on the control of a tail-actuated robotic fish. Experimental results show that the proposed data-driven control approach outperforms a tuned PID (Proportional Integral Derivative) controller and that updating the data-driven model online significantly improves performance in the presence of unmodeled fluid disturbance. This paper is complemented with a video: \href{https://youtu.be/9_wx0tdDta0}{https://youtu.be/9\textunderscore wx0tdDta0}.\end{abstract}


\IEEEpeerreviewmaketitle

\section{Introduction}\label{Sec:: Intro}
\IEEEpubidadjcol
\IEEEPARstart{D}{ynamics} of robotic systems are often unknown, highly nonlinear, and high-dimensional, making real-time control challenging\cite{nmpc_challenges}. Underwater applications represent many of these unmet challenges. In particular, underwater robots are often underactuated (typically by design to reduce weight and cost) and highly nonlinear, and the fluid environments they operate in are difficult to model and time-varying. These challenges call for feedback policies that can learn or adapt online to unmodeled changes \cite{datadrivencontrol} and balance model accuracy and computational efficiency. 

One can draw from many control schemes, including linear quadratic regulator (LQR) \cite{LQR}, linear model predictive control (LMPC) \cite{lmpc}, nonlinear model predictive control (NMPC) \cite{nmpc}, feedback linearization \cite{isidori2013nonlinear}, differential dynamic programming (DDP) \cite{ddp}, sequential action control (SAC) \cite{SAC} and variants of the above \cite{sddp, iLQG, SAC2}. In fact, several of these methods have already been explored in underwater tasks using robotic fish of different morphologies. Researchers have performed maneuvering, speed and orientation control, collision-avoidance, point-to-point navigation as well as velocity and position tracking using a myriad of control schemes, such as PID \cite{Control_Yu, deng2015yaw}, LQR \cite{Maneuverability_Suebsaiprom}, SAC \cite{mamakoukas2016, SAC2}, fuzzy control \cite{controlHorizontal_Kato, wen2011novel}, geometric control \cite{Geo_Morgansen, Trajectory_Morgansen}, sliding mode control \cite{SlidingMode_suebsaiprom}, NMPC \cite{NMPC_maria}, feedback linearization \cite{zhang2018modeling}, backstepping control \cite{BSC_maria} or even a combination of the above \cite{motionPlanning_Saimek, ren2015motion}. However, the aforementioned methods are either system-specific, apply to dynamics with certain structures, or are computationally prohibitive for real-time identification and control of resource-constrained robots. They also typically require full knowledge of the dynamics.

The Koopman operator has recently drawn attention in the robotics community, as it can help address both the difficulty with nonlinearity and the need to incorporate data in the model \cite{koopman_linear_si, koopman_datadrivenapproximation_edmd, kaiser2019data}. Specifically, the Koopman operator propagates a nonlinear system in a linear manner without loss of accuracy by evolving functions of the states \cite{koopman}. The linear representation allows one to control the nonlinear system using tools from linear optimal control \cite{koopman_actuation, koopman_KIC}, which is often easier and faster to implement than nonlinear methods, thus enabling online feedback for high-dimensional nonlinear systems. Beyond the computational speed and the reduction in feedback complexity, the linear representation-based control could lead to better performance compared to a controller that is based on the original nonlinear system \cite{brunton_invariant}. The Koopman operator can also be readily combined with machine learning tools to help learn unknown dynamics from data \cite{koopman_mezic, koopmanism, koopman_mpc, koopman_dmd, koopman_stabilityanalysis, koopman_sindy, koopman_deeplearning, koopman_ian, Bruder_Koopman, koopman_symmetries_salova, dogra2020optimizing}. 
\IEEEpubidadjcol

A downside of the Koopman operator, however, is that, unless a finite-dimensional invariant subspace exists \cite{brunton_invariant}, it is infinite-dimensional. For this reason, recent studies try to obtain a finite-dimensional approximation to the Koopman operator that still captures the dynamics with high fidelity \cite{koopman_datadrivenapproximation_edmd, koopman_KIC}. Koopman-based optimal control applications have successfully implemented such finite-dimensional approximations of the operator for various systems with unknown dynamics \cite{koopman_ian, Bruder_Koopman}. In the trade-off between the dimensionality and the modeling accuracy of the linear representation, these studies face the challenge of finding the minimum number and choice of basis functions for the desired accuracy \cite{koopman_deeplearning}. 

\IEEEpubidadjcol

There has not been a systematic way to address general nonlinear systems; rather, most efforts rely on trial-and-error \cite{koopman_ian,Bruder_Koopman, koopman_symmetries_salova, peitz2019koopman, Koopman_sparsedata, huang2019data} and machine learning tools \cite{koopman_deeplearning, koopman_DNN}, or are system-specific \cite{koopman_linear_si}. Furthermore, there is no method available to bound the modeling error of the finite-dimensional Koopman operators for general nonlinear systems. The only relevant study in \cite{predictiveaccuracy_DMD} analyzes the error bounds of Dynamic Mode Decomposition, closely related to the Koopman operator, for a limited class of systems (parabolic partial differential equations) and with restrictive assumptions on the stability of the identified dynamics.

In this work, we introduce a way of choosing the basis functions and analyze the model accuracy with error bounds. Specifically, we construct the basis functions for the Koopman operator using higher-order derivatives of the nonlinear dynamics, which need not be known; only the derivatives of the tracked states must be available. The error bounds, which depend on the prediction time horizon and the magnitude of the derivatives, can be used to determine the basis functions for the desired level of model accuracy. To our best knowledge, this is the first work that selects basis functions using a systematic methodology and provides an error bound on the accuracy of a Koopman representation for general nonlinear dynamics. To adapt to uncertain or changing dynamics, we obtain the operator using a data-driven, least-squares technique that has a closed-form solution. The linear representation is conducive to various controller designs; in this work, we demonstrate the utility of our online data-driven modeling approach in real-time control using LQR, the gains of which can be obtained with negligible computational cost. 

We validate our approach with simulation and experimental results on the control of a tail-actuated robotic fish and compare it to a PID scheme. Although the tuned PID controller is successful at tracking the desired trajectories, it is outperformed in all tasks by the proposed Koopman-LQR controller. Furthermore, updating the dynamical model in real time significantly improves the performance of Koopman-LQR in the presence of unknown fluid disturbance. 

This paper extends the work in \cite{RSS2019_MamakoukasCastano} in significant ways. The additional contributions include analysis of error bounds on the linear approximation of nonlinear systems, comparison to state-of-the-art modeling methods, realization of online model updates and controller synthesis (thus enabling real-time adaptation), quantitative experimental evaluation, including trials in the presence of unknown fluid disturbance, and extensive comparison of the proposed approach with a PID controller.

The organization of the paper is as follows. Section \ref{sec:: Koopman} reviews the Koopman operator and methods to obtain a data-driven finite-dimensional approximation. Section \ref{sec::Synthesis} describes the proposed synthesis of Koopman basis functions and derives the associated error bounds. Section \ref{sec::Results} evaluates the proposed data-driven modeling scheme on the control of a tail-actuated robotic fish, with and without the presence of flow disturbance. Section \ref{sec:conclusion} summarizes the findings of this paper and discusses ideas for further expanding this work. 

\section{Background} \label{sec:: Koopman}

\subsection{Koopman Operator}\label{sec:: Koopman/KoopmanOperator}
\indent The Koopman operator $\mathcal{K}$ is an infinite-dimensional linear operator that evolves functions of the state $s \in \mathbb{R}^N$ (i.e., $\Psi(s)$, commonly referred to as observables) of a dynamical system. 
Given general nonlinear dynamics of the form 
\begin{align}
 s_{k+1} = F(s_k),
\end{align}
where $F$ is the flow map, the Koopman operator advances the observables with the flow of the dynamics:
\begin{align}
 \mathcal{K}\Psi = \Psi \circ F.
\end{align}
Thus, it advances measurements of the states linearly.
That is,
\begin{align}\label{eq:: Koopmaneq}
\frac{d}{dt}\Psi(s) = \mathcal{K}\Psi(s) \quad \text{and} \quad 
\Psi(s_{k+1}) = \mathcal{K}_d \Psi(s_k),
\end{align}
where $\mathcal{K}$ and $\mathcal{K}_d$ are the continuous-time and discrete-time operators, respectively, related by $\mathcal{K} = \log(\mathcal{K}_d)/\Delta t$ \cite{antsaklis2006linear}. In other words, it allows one to evolve the nonlinear dynamics in a linear setting without loss of accuracy. Contrary to linearizing the dynamics around a fixed point, which leads to inaccurate models away from the linearization point, the Koopman operator evolves a nonlinear system with full fidelity throughout the state space. For a more comprehensive review of the Koopman operator, we refer the reader to \cite{koopmanism}.

Expressing nonlinear systems in a linear manner is a desirable property for many reasons, such as investigating the global stability of a system \cite{koopman_stabilityanalysis}, or extending the local linearization around a point to the whole basin of attraction \cite{koopman_basisofattraction}. In addition to studying the behavior of complex systems, the Koopman framework enables the use of linear optimal control for original nonlinear dynamics. Unfortunately, the infinite-dimensional nature of the Koopman operator makes practical use prohibitive. 

\subsection{Koopman Invariant Subspaces}
There exist nonlinear systems that admit a finite-dimensional linear Koopman representation. Work in \cite{brunton_invariant} analytically derives such Koopman invariant subspaces for nonlinear systems with a specific polynomial structure, whereas the authors in \cite{koopman_kronic, nandanoori2019data, haseli2019efficient,haseli2020Invariant, koopman_invariant_Naoya} identify such spaces from data. In these studies, the authors demonstrate that the LQR control based on the linear representation can outperform LQR control calculated based on the original, nonlinear dynamics. Unfortunately, Koopman invariant subspaces have only been found for a few systems, mentioned above. In fact, there can be no finite-dimensional invariant subspace that includes the states for systems with multiple fixed points \cite{brunton_invariant}. 

In the absence of a finite-dimensional Koopman invariant subspace, a linear propagation of states will induce errors. Regardless, the benefits of a linear model motivate obtaining an approximation to the Koopman operator that will evolve the nonlinear system with acceptable accuracy. Recent studies use data-driven regression schemes to approximate the infinite-dimensional operator $\mathcal{K}$ with a finite-dimensional representation $\tilde{\mathcal{K}}$ \cite{koopman_datadrivenapproximation_edmd, koopman_KIC, Bruder_Koopman}. In this paper, we adopt the least-squares method shown in \cite{koopman_datadrivenapproximation_edmd}, which we detail next. Note that the regression method assumes basis functions that are already known, yet our main contribution is in systematically defining those observables in the first place, which we present in Section \ref{sec::Synthesis}.

\subsection{Data-driven Finite-dimensional Approximation to Koopman Operators}

To obtain an approximation to the Koopman operator, $\tilde{\mathcal{K}}$ $\in \mathbb{R}^{w\times w}$, one can choose a set of observable functions $\Psi(s) = [\psi_{1}(s), \psi_{2}(s), \dots, \psi_{w}(s)]$ $: \mathbb{R}^N \mapsto \mathbb{R}^w$ (which can include the states $s$ themselves) and use data to solve a least-squares minimization problem. To allow for the effect of actuation, \eqref{eq:: Koopmaneq} is modified such that the observables include control terms $u$ as well \cite{koopman_KIC, koopman_ian}. For the discrete-time case, this minimization takes the form 
\begin{align}\label{eq:: Koopman_LS_solution}
\tilde{\mathcal{K}}^*_d = \underset{\tilde{\mathcal{K}}_d}{\operatorname{argmin}}&~ \sum_{k = 0}^{P-1}\frac{1}{2}\lVert \Psi(s_{k+1}, u_{k+1}) - \tilde{\mathcal{K}}_d \Psi(s_k, u_k)\rVert^2,
\end{align}
where $P$ is the number of measurements. Each measurement is a set of an initial state $s_k$, final state $s_{k+1}$, and the actuation applied at the same instants, $u_k$ and $u_{k+1}$, respectively. The above expression has a closed-form solution, given by
\begin{align}\label{eq:: Kd_AG}
\tilde{\mathcal{K}}^*_d = \mathcal{A}\mathcal{G}^\dagger,
\end{align}
where 
\begin{equation}\label{eq::AG}
\begin{aligned}
\mathcal{A} =& \frac{1}{P} \sum_{k = 0}^{P-1} \Psi(s_{k+1}, u_{k+1}) \Psi(s_k, u_k)^T \\
\mathcal{G} =& \frac{1}{P} \sum_{k = 0}^{P-1} \Psi(s_{k}, u_{k}) \Psi(s_k, u_k)^T
\end{aligned}
\end{equation}
and $\dagger$ is the Moore-Penrose pseudoinverse. Note that the time spacing $\Delta t$ between measurements $s_k$ and $s_{k+1}$ must be consistent for all $P$ training measurements. 

The data-driven approximation of the Koopman operator is not inherently different from other system identification techniques. In fact, the Koopman operator can be approximated using any of the standard regression methods, such as ridge or lasso regression \cite{machine_learning, statistical_learning}. More importantly, contrary to standard system identification tools that may try to estimate unknown parameters or, more generally, the nonlinear dynamics of a system \cite{koopman_sindy, system_identification}, the Koopman operator framework places the system identification task in the context of seeking linear transformations of the states, which is useful for control \cite{si_nature, si_equationfree, datadriven} and other purposes, as discussed in Section \ref{Sec:: Intro}.

\subsection{LQR on Koopman Operator}\label{sec:: ControlSynthesis}
Consider a linear system with states $s\in \mathbb{R}^N$, control $u\in \mathbb{R}^M$, and a performance objective \vspace{-0.1cm}
\begin{align}\label{eq::objective}
J = \sum_{k=0}^{\infty} (s_k-s_{des,k})^TQ(s_k-s_{des,k}) + u_k^TRu_k,
\end{align}
where $Q \succeq 0 \in \mathbb{R}^{N\times N}$ and $R \succ 0 \in \mathbb{R}^{M \times M}$ are weights on the deviation from the desired states $s_{des}$ and the applied control, respectively. Further, consider the Koopman representation
\begin{align}\label{eq:: Koopman_discreteDynamics}
 \Psi(s_{k+1}, u_{k+1}) = \tilde{\mathcal{K}}_d \Psi(s_{k}, u_{k}).
\end{align}
For simplicity, we choose $\Psi(s,u) = [\Psi^T_s(s), \Psi^T_u(u)]^T$, where $\Psi_s(s) \in \mathbb{R}^{w_s}$ are the functions that depend on the states $s$ and $\Psi_u(u) \in \mathbb{R}^{w_u}$ are the functions that depend on the input $u$, where $w = w_s + w_u$. Using this notation, we rewrite \eqref{eq:: Koopman_discreteDynamics} as 
\begin{align}\label{eq:: Koopman_affineDynamics}
 \begin{bmatrix}\Psi_s(s_{k+1}) \\ \Psi_u(u_{k+1}) \end{bmatrix} = 
 \begin{bmatrix}
 A & B \\
 C & D
 \end{bmatrix}
 \begin{bmatrix}
 \Psi_s(s_k) \\
 \Psi_u(u_k) 
 \end{bmatrix},
\end{align}
where $A \in \mathbb{R}^{w_s \times w_s}$ and $B \in \mathbb{R}^{w_s \times w_u}$ are submatrices of $\tilde{\mathcal{K}}_d$ that describe the dynamics of the state-dependent functions and change only when $\tilde{\mathcal{K}}_d$ is updated. We use $\Psi_u(u_k) = u_k$ to ensure that control appears linearly in the model such that 
\begin{equation}
 \Psi_s(s_{k+1}) = A \Psi_s(s_k) + B u_k.
\end{equation}

Given the Koopman dynamics \eqref{eq:: Koopman_affineDynamics}, we choose the performance objective
\begin{equation}\label{eq:: J_K}
\resizebox{0.99\hsize}{!}{$
\begin{aligned}
J_{\tilde{\mathcal{K}}} = \sum_{k=0}^{\infty} &(\Psi_s(s_k) - \Psi_s(s_{des,k}))^T Q_{\tilde{\mathcal{K}} }(\Psi_s(s_k) - \Psi_s(s_{des,k}))
+ u_k^TRu_k,
\end{aligned}
$}
\end{equation}
where $Q_{\tilde{\mathcal{K}} } \succeq 0 \in \mathbb{R}^{w_s \times w_s}$ penalizes the deviation from the desired observable functions $\Psi_s(s_{des})$. We let the first $N$ observables be the original states $s$ and set
\begin{align}
Q_{\tilde{\mathcal{K}} } = \begin{bmatrix} Q & 0 \\ 0 & 0 \end{bmatrix},
\end{align}
so that a meaningful comparison can be made with regards to the original nonlinear system and the associated objective function shown in \eqref{eq::objective}. The Koopman representation is conducive to linear quadratic regulator (LQR) feedback of the form
\begin{align} \label{eq:: K_LQR}
u_k = - K_{LQR} (\Psi_s(s_k) - \Psi_s(s_{des,k}) ),
\end{align}
where $K_{LQR} \in \mathbb{R}^{M \times w_s}$, the LQR gains, can be readily calculated from $A, B$ in \eqref{eq:: Koopman_affineDynamics} \cite{SDRE, SDRE_control}. Note that, for given LQR gains, the control is updated using only the functions $\Psi_s(s)$, leading to minimal computation. For more details on the control policy used for the Koopman representation, the reader can refer to \cite{RSS2019_MamakoukasCastano}. Last, we want to emphasize that our approach for synthesizing data-driven Koopman representations can be used with different feedback schemes, such as MPC control.
 
\section{Synthesis of Basis Functions for Error-Bounded Koopman Representation}\label{sec::Synthesis}
This section motivates using higher-order derivatives of nonlinear dynamics to populate the observables of an approximate Koopman operator. The benefits of a derivative-based representation are twofold. First, subject to a finite number of basis functions, it allows one to best capture, locally in time, nonlinear dynamics. For systems that admit a finite-dimensional Koopman invariant subspace, it is straightforward to show that the terms in the observable functions $\Psi(s)$ capture all higher-order derivatives of the original states. This is the reason why the linear representation matches the nonlinear dynamics with no error. This is also true for the invariant subspaces found in \cite{brunton_invariant}, where the Koopman observable functions are populated using the Carleman linearization approach \cite{carleman_original, carleman_nonlinear, carleman_Lie}; for the polynomial systems considered there, the observable functions correspond to the higher-order derivatives of the nonlinear dynamics. When the derivative functions do not span an invariant subspace, populating the observables $\Psi(s_k)$ with higher-order derivatives instead of arbitrary basis functions generates, locally in time, an increasingly (with the order of derivatives) accurate linear representation of the nonlinear dynamics.

Second, the derivative-based representation enables the derivation of error bounds on future predictions. Notably, when the model is entirely data-driven, these error bounds might not be enforced, but offer sound bound estimates that still hold, as we illustrate in Section \ref{subsec:: SimulationResults_Errors}, but which are then dependent on the quality of data used in the data-driven process. To approximate the Koopman operator, so far studies have largely focused on data-driven methods of the form in \eqref{eq:: Koopman_LS_solution} that consider only the local error across one time step, that is
\begin{align}
 \Psi(s_{k+1}, u_{k+1}) - \tilde{\mathcal{K}}_d \Psi(s_k, u_k),
\end{align} as used in \eqref{eq:: Koopman_LS_solution}. Another measure of the accuracy of the Koopman representation is the global error, over an arbitrarily long time window across all time steps $m >0 $ (see Fig. \ref{fig::LocalvsGlobal}), that is
\begin{align}
 \Psi(s_{k+m}, u_{k+m}) - \tilde{\mathcal{K}}_d^{m} \Psi(s_k, u_k).
\end{align}

The derivative-based linear embedding methodology presented in this work enables the computation of global error bounds of the Koopman representation. By exploiting the accuracy properties of Taylor expansions, we can synthesize Koopman basis functions that bound the model error for any particular order of linear representation. The error bounds in turn allow one to select the lowest-order representation that meets a desired accuracy. This analysis is presented next.

The proposed linear embedding method does not require knowledge of the dynamics; instead, it requires only that the time derivatives of the system states of interest be available. The values of the derivatives can be either evaluated, using knowledge of the dynamics equations, or numerically estimated from state measurements (using finite differencing or other methods \cite{derivatives_estimation}). The method can be used 

\begin{itemize}
 \item for known nonlinear dynamics: each state derivative is analytically derived from the dynamics equation and constitutes a basis function for the Koopman operator (one basis function per derivative)
 \item for dynamics whose structure is known but coefficients might be unknown or changing, as we illustrate in Section \ref{SubSection:: Synthesis}: derivatives are analytically derived from each term that appears in the dynamics equation; each term that is computed constitutes a separate basis function for the Koopman operator (at least one basis function per derivative)
 \item for completely unknown dynamics, as we illustrate in \ref{subsec:: SimulationResults_Errors}: each state derivative is numerically calculated and constitutes a basis function for the Koopman operator (one basis function per derivative).
\end{itemize}

\subsection{Error Bounds of Derivative-Based Koopman Operators }\label{Section: ErrorBounds}

\begin{figure}
	\centering
	\includegraphics[width=1\columnwidth, keepaspectratio= true]{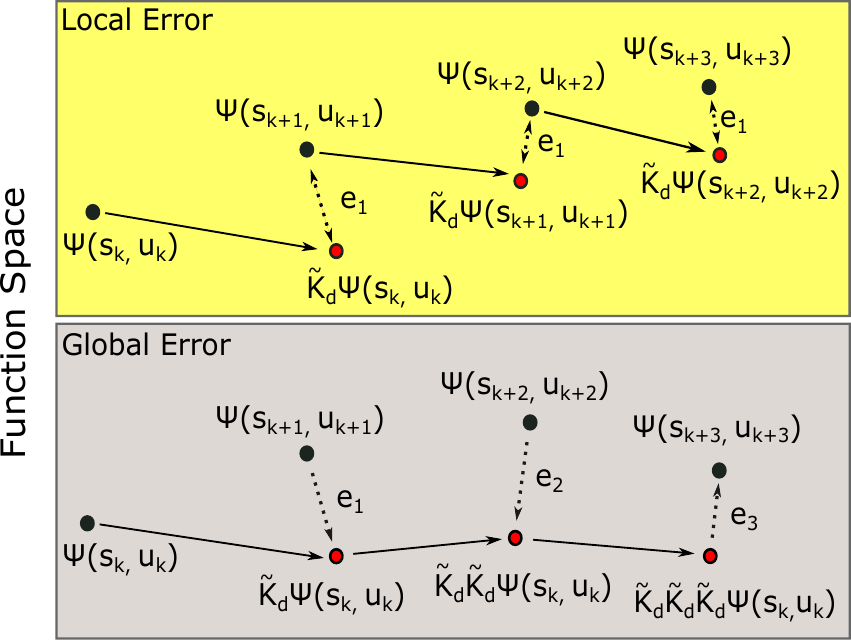}
	\caption{Local and global errors induced by approximate Koopman operators. The local error is the error induced by the operator across one step, assuming no error in the initial conditions. The global error is the total deviation away from the true states across multiple steps.}\label{fig::LocalvsGlobal}
\end{figure}

The evolution of a nonlinear function $f(t)$ that is continuously differentiable up to $n$th order can be approximated with a Taylor series as
\begin{equation}
\resizebox{0.99\hsize}{!}{$
	\begin{aligned}
\tilde{f}(t_{k+1}) =& f(t_k) + f'(t_k) \cdot (t_{k+1} - t_{k}) \\
&+ \frac{f''(t_k)}{{2!}} \cdot (t_{k+1} - t_{k})^2 + \dots + \frac{f^{(n)}(t_k)}{n!} (t_{k+1} - t_{k})^n,
\end{aligned}
$}
\end{equation}
where tilde $(\tilde{\cdot})$ denotes the predicted value of a function, and not its true value.\footnote{We assume that the true values of the function $f$ and its derivatives are known at time $t_k$. Uncertainty about the initial values can be readily included in the formula of the global error, later shown in this paper.} To keep the algebraic expressions compact, let $t_{k+1} - t_k = \Delta t$ and $f_k^{(i)}, \triangleq f^{(i)}(t_k),~\forall~i\in\mathbb{Z}\,\cap\,[1,n]$, which simplifies the above expression to 
\begin{align}\label{eq:: Taylorexpansion}
\tilde{f}_{k+1} = f_k + f'_k \cdot \Delta t + f''_k \cdot \frac{\Delta t^2}{2!} + \dots + f_k^{(n)} \frac{\Delta t^n}{n!}.
\end{align}
Propagating a function using its derivatives allows one to use the accuracy of the Taylor series to characterize the error in the evolution of a function across one time step $\Delta t$. The local error induced by a Taylor series approximation using up to $n$ derivatives across one time step is $R_n(k) = f_{k+1} - \tilde{f}_{k+1}$. This error is calculated using Lagrange's remainder formula: 
\begin{align}\label{lagrange's}
R_n(k) = \frac{f^{(n+1)}_c}{(n+1)!} \Delta t^{n+1},
\end{align}
where $f^{(n+1)}_c \triangleq f^{(n+1)}(c)$ is the time derivative of order $n+1$ evaluated at some time $c \in [t_k, t_{k+1}]$. If there exists a positive real number $L$ such that $|f^{(n+1)}_c|$ $\le L$ for all $c \in [t_k, t_{k+1}]$, then the upper error bound in \eqref{lagrange's} becomes
\begin{align}\label{eq:: local_error_bound}
|R_n(k)| \le \frac{L}{(n+1)!} \Delta t^{n+1}.
\end{align}

For the purpose of applying linear control synthesis tools to linear representations of nonlinear dynamics, we bring the Taylor approximation in \eqref{eq:: Taylorexpansion} to a linear matrix form:
\begin{align}\label{eq:: Taylor_matrix}
\underbrace{
\begin{pmatrix}\tilde{f}_{k+1} \\ \vphantom\vdots\tilde{f}'_{k+1} \\ \vphantom\vdots\tilde{f}''_{k+1} \\ \vdots \\ \vphantom\vdots\tilde{f}^{(n)}_{k+1} \end{pmatrix}}
_{\Psi(s_{k+1})}
\approx
\underbrace{\begin{pmatrix} 1 & \Delta t & \dfrac{{\Delta t}^2}{2} & \cdots & \dfrac{\Delta t^n}{n!} \\
0 & 1 & \Delta t & \cdots & \dfrac{\Delta t^{n-1}}{(n-1)!}\\
0 & 0 & 1 & \cdots & \dfrac{\Delta t^{n-2}}{(n-2)!}\\
\vdots & \vdots & \vdots & \ddots & \vdots \\
\vphantom\vdots0 & 0 & 0 & \cdots & 1
\end{pmatrix}}_{\tilde{\mathcal{K}}_d}
\underbrace{\begin{pmatrix} f_{k} \\ \vphantom\vdots f'_{k} \\ \vphantom\vdots f''_{k} \\\vdots \\ \vphantom\vdots f^{(n)}_{k}\end{pmatrix}}_{\Psi(s_k)}.
\end{align}
For a fixed $\Delta t$, expression \eqref{eq:: Taylor_matrix} resembles \eqref{eq:: Koopmaneq}, where the derivatives of the function $f_k$ are the observables $\Psi(s_k)$. When representing the Taylor series expansion, the derivative functions $f^{(i)}$ are known at time step $t_k$ and approximated at time $t_{k+1}$ by $\tilde{f}^{(i)}$. When training a Koopman operator, pairs of measurements of the states $s_k$ and $s_{k+1}$ are used to evaluate the basis functions at the corresponding time steps: $\Psi(s_k)$ and $\Psi(s_{k+1})$. Note that, in \eqref{eq:: Taylor_matrix}, all derivatives of $f_k$ are assumed to be different functions. The analytical expression \eqref{eq:: Taylor_matrix} is equivalent to a Taylor series expansion \eqref{eq:: Taylorexpansion} across one time step $\Delta t$ for all the basis functions. Therefore, the same error analysis \eqref{eq:: local_error_bound} applies to each observable in \eqref{eq:: Taylor_matrix}. 

When propagating a function across multiple time-steps using \eqref{eq:: Taylor_matrix}, the observable functions are themselves numerically propagated instead of being evaluated with measurable states at each time step, as is the case for a typical integration scheme. As a result, error accumulates not only in approximation of the original function $f$, but in the other observables as well. To track the error in the original $f$, therefore, it is necessary to be able to model the error in all of the basis functions. Using the accuracy of the Taylor series structure, we are able to model the error on every basis function and ultimately bound the model error in $f$. 

\begin{theorem}\label{th:: global_error_bound}
Consider a general nonlinear function $f(t)$ that is continuously differentiable up to order $n$. Propagating $f(t)$ and its first $n$ derivatives using the Taylor-based linear representation \eqref{eq:: Taylor_matrix} induces an error in $f(t)$ that is given by 
\begin{align}\label{eq:: global_error_bound}
e_k =& \sum_{i = 1}^{k-1}\sum_{j = 1}^{n} e_i^{(j)} \frac{\Delta t^j}{j!} + \sum_{i = 0}^{k-1}f_{i, i+1}^{(n+1)}\frac{\Delta t^{n+1}}{(n+1)!},
\end{align} 
where $n\,\in \mathbb{Z}^{\ge 0}$ is the number of derivative basis functions used, $k\,\in \mathbb{Z}^{\ge 1}$ is the number of time steps into the future, and, from Lagrange's remainder formula \eqref{lagrange's}, $f^{(n+1)}_{i,i+1}$ is the $n+1$th time derivative of function $f$ evaluated at some time $t \in [t_i, t_{i+1}]$. The error bound is given by
\begin{align}\label{eq:: GlobalError_simplified}
|e_k| \le& \frac{T^{n+1}}{(n+1)!}|f^{(n+1)}_{max}|,
\end{align}
where $T \triangleq k \Delta t$ is the prediction time horizon and $|f^{(n+1)}_{max}|$ is the maximum magnitude of the $n+1$th derivative.
\end{theorem}
\begin{proof}
For the derivation of the error expression \eqref{eq:: global_error_bound}, see Appendix \ref{App:: GlobalError}. For the derivation of the error bound formula \eqref{eq:: GlobalError_simplified}, see Appendix \ref{App:: ErrorBounds}.
\end{proof}

To the best of our knowledge, this is the first work that provides prediction error bounds on the accuracy of a Koopman representation for general nonlinear dynamics. Prediction error bounds based on Taylor series had previously only been derived for a single step, and not for an arbitrary number of time steps into the future, as we derive in this paper.

The error bound \eqref{eq:: GlobalError_simplified} is associated with the Koopman representation \eqref{eq:: Taylor_matrix} for the dynamics of a single function $f$. The same methodology can be used to propagate multiple states of a system with coupled dynamics. Specifically, a system with states $s(t)$ and general nonlinear dynamics $\dot{s}(t) = g(s(t)) \in \mathbb{R}^N$ that are continuously differentiable up to order $n$
\begin{align}
\frac{d}{dt}
\begin{bmatrix}
 s_1 \\
 s_2 \\
 \vdots \\
 s_N
 \end{bmatrix}
 =&
 \begin{bmatrix}
 g_1(s) \\
 g_2(s) \\
 \vdots \\
 g_N(s)
 \end{bmatrix}
\end{align}
can be propagated in discrete time as
\begin{equation}
\begin{bmatrix}
 s_{1,k+1} \\ s_{2,k+1} \\ \vdots \\ s_{N,k+1} 
\end{bmatrix}
 =
\begin{bmatrix}
s_{1,k} + g_{1,k} \cdot \Delta t + \dots + g_{1,k}^{(n_1)} \frac{\Delta t^{n_1+1}}{{(n_1+1)}!} \\
s_{2,k} + g_{2,k} \cdot \Delta t + \dots + g_{2,k}^{(n_2)} \frac{\Delta t^{n_2+1}}{{(n_2+1)}!} \\
\vdots \\
s_{N,k} + g_{N,k} \cdot \Delta t + \dots + g_{N,k}^{(n_N)} \frac{\Delta t^{n_N+1}}{{(n_N+1)}!}
\end{bmatrix},
\end{equation}
where $n_j~\text{for}~j\in\mathbb{Z}\,\cap\,[1,N]$ indicates the highest-order of derivatives of $g_j$ used to propagate the $j$th state of the original dynamics (which does not have to be the same for all states). The above expression can be rewritten in a linear form similar to \eqref{eq:: Taylor_matrix}

\begin{equation}\label{eq:: TaylorMatrix_Manystates}
\resizebox{0.99\hsize}{!}{$
\underbrace{\left[ \begin{array}{c}
 s_{1,k+1} \\ g_{1,k+1} \\ \vdots \\ g_{1,k+1}^{(n_1)} \\ [1.0ex] \hdashline[2pt/2pt] \\ [-2.0ex] s_{2,k+1} \\ g_{2,k+1}\\ \vdots \\ g_{2,k+1}^{(n_2)} \\ [1.0ex] \hdashline[2pt/2pt] \\ [-2.0ex] \vdots \\ [1.0ex] \hdashline[2pt/2pt] \\ [-2.0ex] s_{N,k+1} \\ g_{N,k+1} \\ \vdots \\ g^{(n_N)}_{N,k+1}
\end{array} \right]}_{\Psi(s_{k+1})}
 =
\underbrace{\left[ 
\begin{array}{cccc}
 \begin{array}{c@{}c@{}c@{}c} T(n_1) \end{array} & \mathbf{0} & \cdots & \mathbf{0}\\
 [2.0ex]
 \hdashline[2pt/2pt] \\
 \mathbf{0} & \begin{array}{c@{}c@{}c@{}c} T(n_2) \end{array} & \cdots & \mathbf{0}\\
 [2.0ex]
 \hdashline[2pt/2pt] \\ 
 \mathbf{0} & \mathbf{0} & \ddots & \vdots \\ 
 [2.0ex]
 \hdashline[2pt/2pt] \\
 \mathbf{0} & \mathbf{0} & \cdots & \begin{array}{c@{}c@{}c@{}c} T(n_N) \end{array} \\ [0.5ex]
\end{array}\right]}_{\tilde{\mathcal{K}}_d}
\underbrace{\left[ \begin{array}{c}
 s_{1,k} \\ g_{1,k} \\ \vdots \\ g_{1,k}^{(n_1)} \\ [1.0ex] \hdashline[2pt/2pt] \\ [-2.0ex] s_{2,k} \\ g_{2,k}\\ \vdots \\ g_{2,k}^{(n_2)} \\ [1.0ex] \hdashline[2pt/2pt] \\ [-2.0ex] \vdots \\ [1.0ex] \hdashline[2pt/2pt]\\ [-2.0ex] s_{N,k} \\ g_{N,k} \\ \vdots \\ g^{(n_N)}_{N,k}
\end{array} \right]}_{\Psi(s_{k})},\\
$}
\end{equation}
where 
\begin{equation}
T(n_j) = 
\left[ 
 \begin{array}{cccc}
 1 & \Delta t & \cdots & \dfrac{\Delta t^{n_j+1}}{(n_j+1)!} \\
 0 & 1 & \cdots & \dfrac{\Delta t^{n_j}}{n_j!} \\
 \vdots & \vdots & \ddots & \vdots \\
 0 & 0 & \cdots & 1 \\
 \end{array}\right]~\text{for}~j\in\mathbb{Z}\,\cap\,[1,N].
\end{equation}

Note how \eqref{eq:: TaylorMatrix_Manystates} is grouped into submatrices that propagate independently each state and its higher-order derivatives. The basis functions are the states and their derivatives. Propagating a nonlinear system with states $s$ and nonlinear dynamics $g(s(t))$ using \eqref{eq:: TaylorMatrix_Manystates} is equivalent to propagating each state separately using \eqref{eq:: Taylor_matrix} and thus induces, for each state, an error given by an expression similar to \eqref{eq:: global_error_bound}. 

The formulation in \eqref{eq:: TaylorMatrix_Manystates} uses basis functions that depend, for simplicity, only on the state $s$. When working with a system that has control inputs, one can treat controls $u$ in a similar fashion and calculate its higher-order derivatives, by introducing $u$ as dummy states that are the derivatives of the control input, a common practice \cite{koopman_kronic}.

\begin{corollary}
Consider general nonlinear dynamics $\dot{s}(t) = g(s(t))$ that are continuously differentiable up to order $n$. Propagating $s(t)$ using \eqref{eq:: TaylorMatrix_Manystates} induces a bounded error on state $s_i,~i\in\mathbb{Z}\,\cap\,[1,N],$ given by
\eqref{eq:: global_error_bound}, where $g(s(t))^{(i-1)} = f(t)^{(i)}$. 
\end{corollary}
\begin{proof}
Consider the propagation of each state $s_i$ and its derivative functions $g_i^{(\cdot)}$. Let each $s_i$ be a function $f_i$ whose $n_i^{th}$ derivative is $f^{(n_i)}$. Then, from Theorem \ref{th:: global_error_bound}, propagating each state $s_i \triangleq f_i$ with \eqref{eq:: Taylor_matrix} induces an error given by \eqref{eq:: global_error_bound}.
\end{proof}
The error bound \eqref{eq:: GlobalError_simplified} allows one to calculate the maximum possible error in each system state when propagating it with the fixed linear matrix \eqref{eq:: TaylorMatrix_Manystates}. As a result, the bound can be used to determine the desired number of derivatives that are needed for each state that would generate minimal error given a fixed prediction horizon $T$ and subject to the nonlinear dynamics. Alternatively, the error bound can also be used to compute the maximum length of the prediction time horizon for which the state error is bound to remain under a threshold given a set number of derivative basis functions.

Because, in general, there is no closure of the higher-order derivatives and the series has to be truncated, the analytical expression \eqref{eq:: Taylor_matrix} would only lead to an approximate Koopman operator, as is commented in \cite{brunton_invariant}. Note, for example, that the highest derivatives in \eqref{eq:: TaylorMatrix_Manystates} are not updated at all. For this reason, we use data-driven techniques to obtain a $\tilde{\mathcal{K}}_d$ that more accurately advances all of the basis functions than the analytical expression. On the other hand, the error bound \eqref{eq:: GlobalError_simplified} applies to a linear propagation of nonlinear dynamics using \eqref{eq:: TaylorMatrix_Manystates} and therefore is no longer guaranteed when a data-driven operator is used instead. Nevertheless, it can still serve to measure how amenable nonlinear dynamics are to a linear representation by revealing the relationship between the magnitude and order of the derivatives. Furthermore, empirically, the data-driven model does resemble the Taylor-series structure \eqref{eq:: TaylorMatrix_Manystates} such that the error bounds remain relevant.\footnote{Although the data-driven solution is not guaranteed to bound all local errors within the Taylor series accuracy, given a training dataset that is a representative part of the state space, solutions that largely deviate from the Taylor-series structure in \eqref{eq:: Taylor_matrix} would generate large local errors in parts of the state space and thus be avoided by the least-squares solution \eqref{eq:: Kd_AG}.} Using simulation results, we next verify the similarity of the data-driven operator to the analytical expression in \eqref{eq:: Taylor_matrix}, as well as the validity of the error bounds.

\subsection{Error Bound Estimation Using Data-Driven Operator}\label{subsec:: ErrorBoundEstimationUsingDataDrivenOperator}
The error bound formula \eqref{eq:: GlobalError_simplified}, derived for a linear representation of \eqref{eq:: TaylorMatrix_Manystates}, remains relevant to a data-driven operator, when the latter has similar structure, i.e., small Frobenius distance, to the Taylor-series form \eqref{eq:: Taylor_matrix}. On the other hand, since an operator computed from data may not take exactly the form of (\ref{eq:: TaylorMatrix_Manystates}), the bounds shown in (\ref{eq:: GlobalError_simplified}) are not strictly enforced, but offer what we refer to as sound bound estimates in the remainder of the paper. To calculate the bound estimates, one needs to compute $|f_{max}^{(n+1)}|$, the magnitude of the lowest-order derivative of the system states that is not used in the basis functions. When dynamics are known, this value can be calculated numerically. Alternatively, as we show next, one can exploit the Taylor-series structure of the data-driven operator to estimate the error bounds beyond the training set that has been used to generate the Koopman operator even when there is no knowledge of the dynamics. 

Specifically, given a linear representation that approximates the Taylor-series structure \eqref{eq:: Taylor_matrix}, the local error across one time step induced by the data-driven model can be described by the Taylor series accuracy. Thus, using \eqref{eq:: GlobalError_simplified}, the error across one time step ($k = 1$) of a function $f$ can be written as
\begin{align}
|e_1| \le |f_{max}^{(n+1)} \frac{\Delta t^{(n+1)}}{(n+1)!}|
=& |f_{max}^{(n+1)}| \frac{\Delta t^{(n+1)}}{(n+1)!},
\end{align}
where $e_1$ is available from the data-driven training process \eqref{eq:: Koopman_LS_solution}.
Let $|e_1|_{max}$ be the maximum local error, i.e., $|e_1|_{max} \ge |e_1|$. Then, when the training data set is large enough, one can get
\begin{align}
|e_1|_{max} \approx |f_{max}^{(n+1)}| \frac{\Delta t^{(n+1)}}{(n+1)!},
\end{align}
which is rearranged to
\begin{align}\label{eq::datadriven_upperbound}
|f^{(n+1)}_{max}|\approx |e_1|_{max} \frac{(n+1)!}{\Delta t^{(n+1)}}.
\end{align}
In short, we use the maximum error across one time step from the training process to estimate the term $|f_{max}^{(n+1)}|$, which in turn, using \eqref{eq:: GlobalError_simplified}, allows us to estimate $|e_k|$, the error bound after $k$ time steps. Alternatively, when no analytical model of the dynamics is available, the value $|f^{(n+1)}_{max}|$ can also be estimated numerically using measurements of $f$.

\subsection{Synthesis of Derivative-Based Koopman Observables with Structural Knowledge of Dynamics}\label{SubSection:: Synthesis}
The derivative-based approach proposed in this work populates Koopman observables with the system states $s$ and their derivative functions. Each derivative is a separate function that can be computed from the analytical expression when dynamics are fully known, or numerically estimated from measurements when no model exists. In this subsection, we show how we construct the basis functions to exploit structural knowledge of dynamics that have unknown coefficients. 

For simplicity, we assume that the dynamics of each system state depend on a single term, i.e., a nonlinear function multiplied by a coefficient; the case of having multiple such terms can be handled similarly. In particular, consider a nonlinear system with states $s \in \mathbb{R}^{N}$ and dynamics
\begin{align}
\frac{d}{dt}
\begin{bmatrix}
 s_1 \\
 s_2 \\
 \vdots \\
 s_N
 \end{bmatrix}
 =&
 \begin{bmatrix}
 c_1 g_1(s) \\
 c_2 g_2(s) \\
 \vdots \\
 c_N g_N(s)
 \end{bmatrix},
\end{align}
where $c_i$$, ~i\in\mathbb{Z}\,\cap\,[1,N],$ are unknown coefficients and $g_i(s)$ are nonlinear functions of the states $s$. The second-order time derivatives of the states $s$ are 
\begin{align}
\frac{d^2}{dt^2}
\begin{bmatrix}
 s_1 \\
 s_2 \\
 \vdots \\
 s_N
 \end{bmatrix}
 =&
 \begin{bmatrix}
 c_1 g_1'(s) \\
 c_2 g_2'(s) \\
 \vdots \\
 c_N g_N'(s)
 \end{bmatrix},
\end{align}
where $g_i^\prime(s)$ denotes the time derivative of $g_i$, and thus
\begin{align}
 c_i g'_i(s) = c_i (\frac{\partial g_i}{\partial s_1} c_1g_1 + \dots + \frac{\partial g_i}{\partial s_N} c_Ng_N) ~ \text{for}~i\in \mathbb{Z}\,\cap\, [1,N].
\end{align}
For ease of discussion, we limit the analysis to the first two time derivatives, but the same process can continue to generate higher-order derivatives and, thus, additional basis functions. 

Using the states $s_i$, the first-order derivatives $g_i$ and the individual terms $\dfrac{\partial g_i}{\partial s_j}g_j$ that appear in $g'_i(s)$, where $i, j \in \mathbb{Z}\,\cap\,[1,N]$, we populate the basis functions of the Koopman matrix. In discrete time, the states are then propagated with 
\begin{equation}
\begin{bmatrix}
 s_{1,k+1} \\ s_{2,k+1} \\ \vdots \\ s_{N,k+1} 
\end{bmatrix}
 =
\begin{bmatrix}
s_{1,k} + c_1 g_{1,k} \cdot \Delta t + c_1 g'_{1,k} \dfrac{\Delta t^{2}}{{2}!} \\
s_{2,k} + c_2 g_{2,k} \cdot \Delta t + c_2 g'_{2,k} \dfrac{\Delta t^{2}}{{2}!} \\
\vdots \\
s_{N,k} + c_N g_{N,k} \cdot \Delta t + c_N g'_{N,k} \dfrac{\Delta t^{2}}{{2}!}
\end{bmatrix}.
\end{equation}
Substituting for the $g'_i(s)$ terms, we show the expected form for a single state $s_1$:
\begin{equation}
\resizebox{0.99\hsize}{!}{$
\underbrace{\left[ \begin{array}{c@{}}
 s_{1,k+1} \\ g_{1,k+1} \\ \{\dfrac{\partial g_1}{\partial s_1}g_1\}_{k+1} \\ \{\dfrac{\partial g_1}{\partial s_2}g_2\}_{k+1} \\ \vdots \\ \{\dfrac{\partial g_1}{\partial s_N}g_N\}_{k+1} \\ [1.0ex] \end{array} \right]}_{\Psi(s_{k+1})}
 =
\underbrace{\left[ 
\begin{array}{cccccc@{}}
1 & c_1 \Delta t & c_1^2 \dfrac{\Delta t^2}{2} & c_1c_2 \dfrac{\Delta t^2}{2} & \dots & c_1c_N \dfrac{\Delta t^2}{2} \\
0 & 1 & c_1 \Delta t & c_2\Delta t & \dots & c_N \Delta t \\
0 & 0 & 1 & 0 & \dots & 0 \\
0 & 0 & 0 & 1 & \dots & 0\\
\vdots & \vdots & \vdots & \vdots & \ddots & \vdots \\
0 & 0 & 0 & 0 & \dots & 1
\end{array}\right]}_{\tilde{\mathcal{K}}_d}
\underbrace{\left[ \begin{array}{c@{}}
 s_{1,k} \\ g_{1,k} \\ \{\dfrac{\partial g_1}{\partial s_1}g_1\}_{k} \\ \{\dfrac{\partial g_1}{\partial s_2}g_2\}_{k}\\ \vdots \\ \{\dfrac{\partial g_1}{\partial s_N}g_N\}_{k} \\ [1.0ex] 
\end{array} \right]}_{\Psi(s_{k})}.\\
$}
\end{equation}
As before, to improve the accuracy of the linear representation, we use data to approximate a Koopman operator.

\subsection{Assessment of Error Bound Estimates}\label{subsec:: SimulationResults_Errors}
Using the single pendulum system, we demonstrate the error bound estimates for the data-driven linear approximation, both when the dynamics are known and when they are unknown. In particular, we show that $|f^{(n+1)}_{max}|$ can be computed using the dynamics equations (model-based estimate) when those are available, or approximated using the residue error in the training process when the dynamics are unknown (data-driven estimate), as explained in Section \ref{subsec:: ErrorBoundEstimationUsingDataDrivenOperator}. In both cases, once $|f^{(n+1)}_{max}|$ is calculated, the error bounds are computed using \eqref{eq:: GlobalError_simplified}. The states are $s = [\theta, \omega]^T$ and dynamics are given by $\dot s = [\omega, \frac{g}{l} \sin(\theta) + u]^T$, where $g = 9.81~\nicefrac{\textrm{m}}{\textrm{s}}^2$ is the gravitational constant, $l = 1~\textrm{m}$ is the pendulum length, and $u$ is the control. 

For the model-based estimate of the error bounds, the maximum magnitude of the $n$th derivative is (for each state) computed numerically by maximizing the symbolic expression over the domain of the state space that is used for training. For the data-driven estimate, the maximum magnitude is computed using \eqref{eq::datadriven_upperbound} based on the training error. Note that, when propagating a data-driven operator instead of \eqref{eq:: TaylorMatrix_Manystates}, the error bound estimates may in theory be violated. 

\begin{figure}
\centering
 \begin{subfigure}[t]{0.98\columnwidth}
 \centering
 \includegraphics[width = \columnwidth, height = 0.40\textheight, keepaspectratio = true]{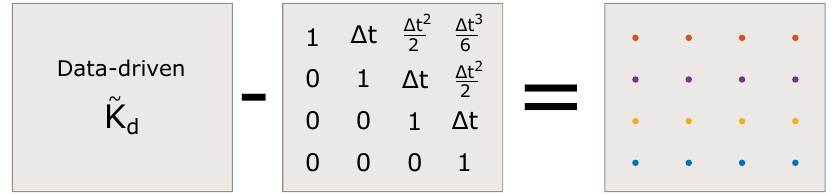}
 \caption{}
 \end{subfigure}%
 \\
 \begin{subfigure}[t]{0.495\columnwidth}
 \centering
 \includegraphics[width = \columnwidth, keepaspectratio = true]{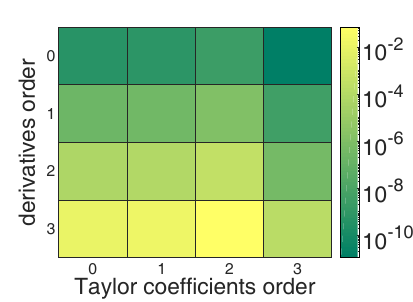}
 \caption{}\label{fig:: HeatMap}
 \end{subfigure} \hfill
 \begin{subfigure}[t]{0.49\columnwidth}
 \centering
 \includegraphics[width = \columnwidth, keepaspectratio = true]{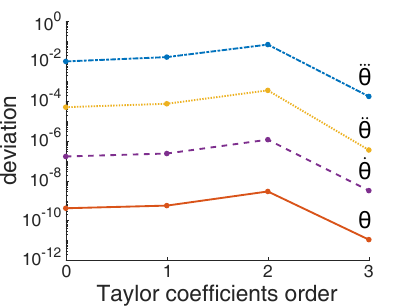}
 \caption{}\label{fig:: DerivativesComparisons}
 \end{subfigure}
 \caption{The deviation of the data-driven Koopman operator from the Taylor-based matrix \eqref{eq:: Taylor_matrix} for the single pendulum system, where the derivative basis functions are constructed analytically from the known dynamics. Fig. \ref{fig:: HeatMap} shows that the non-zero coefficients (upper triangle) of the linear Taylor expansion are accurately recovered from the data-driven operator. The zero coefficients (lower triangle) are replaced by small values that help minimize the least\textcolor{blue}{-}squares error for the part of the state space used in the training set. The deviation differs by orders of magnitude across the basis functions, as seen in Fig. \ref{fig:: DerivativesComparisons}. As expected, the deviation is smallest for $\theta$, as it is the one with the highest number of derivatives used in the basis functions.}\label{fig::DatadrivenKoopman_deviation}
\end{figure}

\begin{figure*}
	\begin{subfigure}[b]{0.325\linewidth}
		\centering
		\includegraphics[width=\linewidth,height = 0.6\linewidth, keepaspectratio = true]{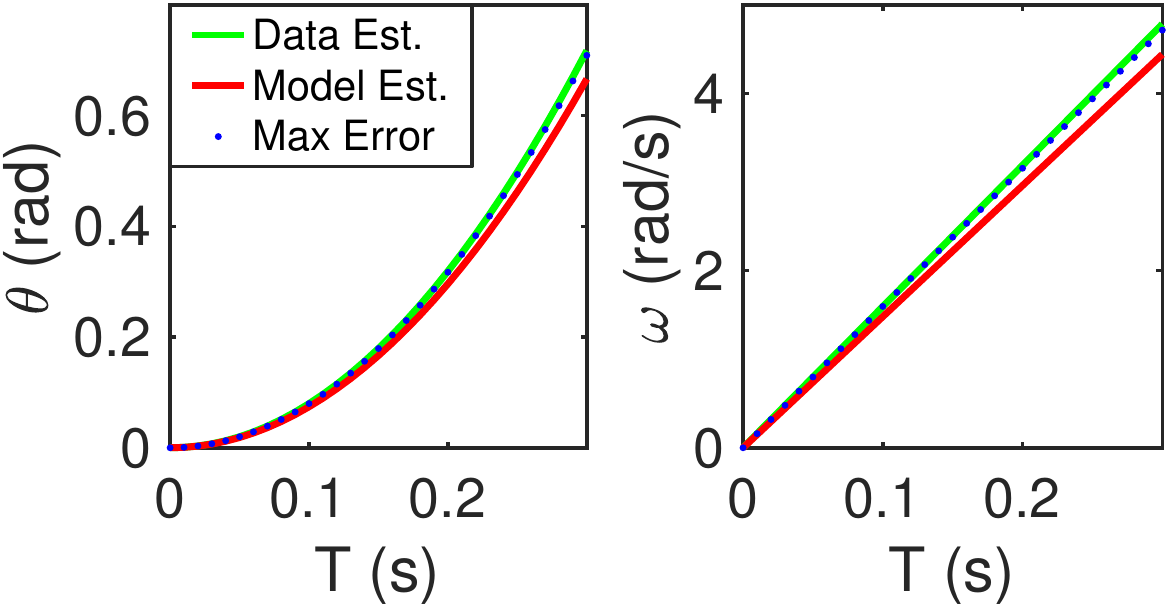} 
		\caption{$n = 1$} 
	\end{subfigure}
	\hfill 
	\begin{subfigure}[b]{0.325\linewidth}
		\centering
		\includegraphics[width=\linewidth,height = 0.6\linewidth, keepaspectratio = true]{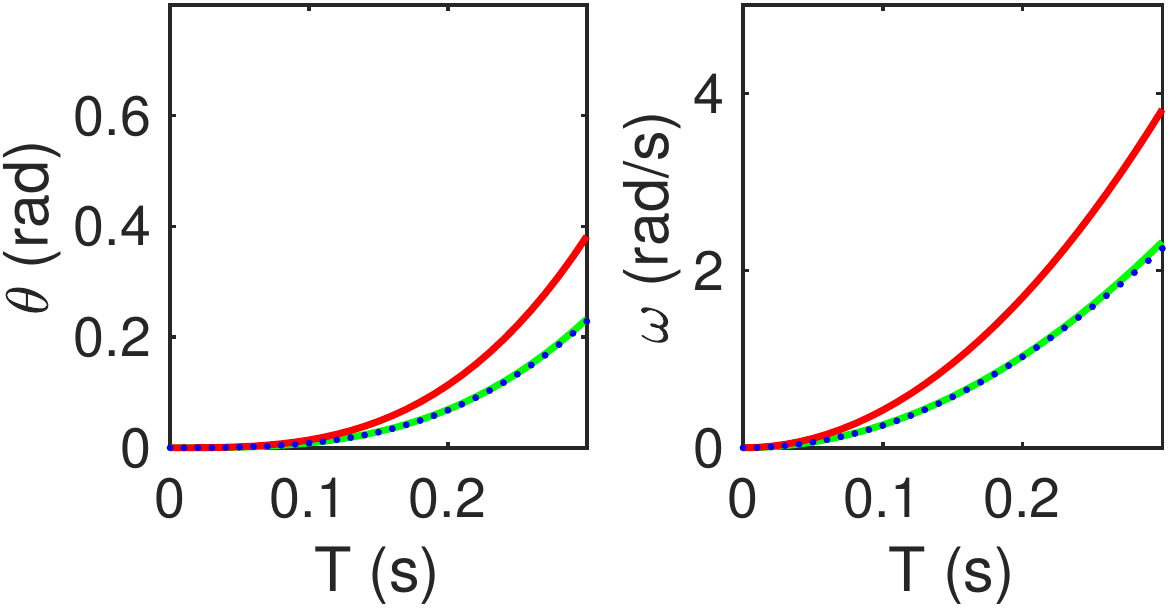} 
		\caption{$n = 2$} 
	\end{subfigure} \hfill
	\begin{subfigure}[b]{0.325\linewidth}
		\centering
		\includegraphics[width=\linewidth,height = 0.6\linewidth, keepaspectratio = true]{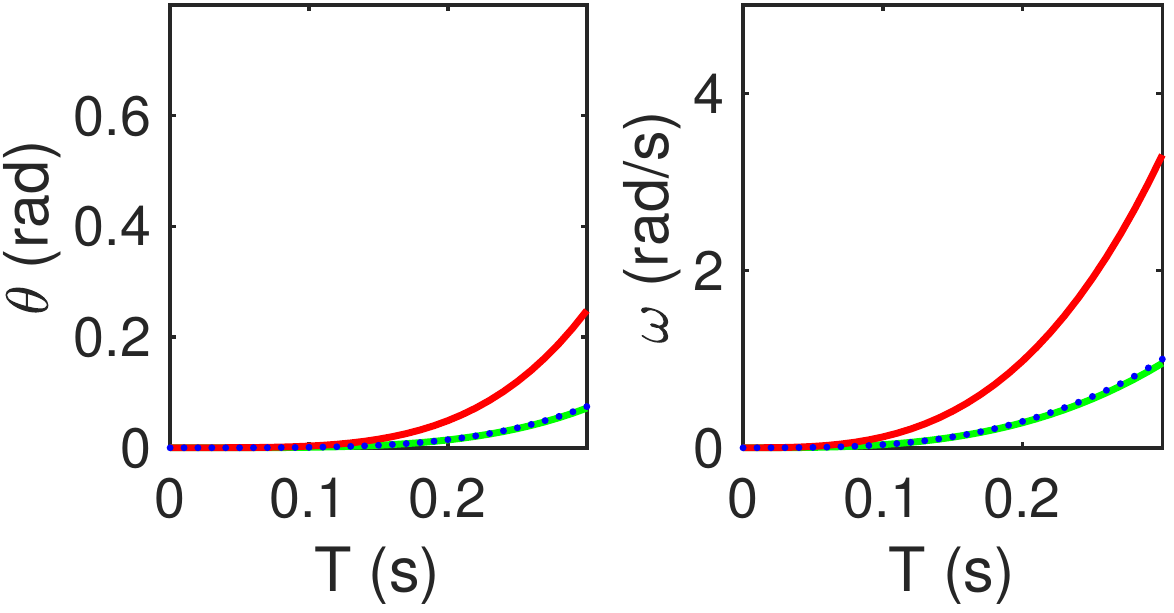} 
		\caption{$n = 3$} 
	\end{subfigure} \hfill
	\caption{Simulated error bound estimates and actual error bounds for the single pendulum system as a function of the prediction horizon and for increasing orders of derivatives used as Koopman basis functions. The derivative basis functions are constructed analytically from the known dynamics. Both error bound estimates are calculated using \eqref{eq:: GlobalError_simplified}, but differ in how they compute $|f^{(n+1)}_{max}|$. \textbf{Data Est.} is the model-free error bound estimate and uses the data-driven Koopman operator and \eqref{eq::datadriven_upperbound} to compute $|f^{(n+1)}_{max}|$; \textbf{Model Est.} is the model-based error bound estimate and uses the analytical dynamics equations to compute $|f^{(n+1)}_{max}|$; \textbf{Max error} is the measured largest deviation as a function of time between the actual value of the state and the one predicted by the data-driven Koopman operator across all trajectories that evolve from randomly sampled initial conditions. Results are shown for three different orders of derivatives of $\theta$. Note that state $\theta$ has always one more derivative than $\omega$. The data-driven bound estimates and actual errors can be generated for
	different parameter choices using a Jupyter notebook at \url{https://colab.research.google.com/drive/1EPX1XVUHr9gix-pZD_3Ydw7Npzz9n3Jj}.}
	\label{fig:: ErrorBounds} 
\end{figure*}

We sample and forward-simulate 5000 initial states $s_0$ for $\Delta t = 0.01$~s to obtain a Koopman operator $\tilde{\mathcal{K}}_d$ via (\ref{eq:: Kd_AG}), which we then use on a different randomly selected set of 5000 states to propagate the dynamics for a time horizon $T$. In both the training and the testing sets, uniform distributions of the initial states $\mathcal{U}_{\theta_0}$($-2\pi$\,\si{\radian}, $2\pi$\,\si{\radian}) and $\mathcal{U}_{\omega_0}$($-$\SI[per-mode=symbol]{5}{\radian\per\second}, \SI[per-mode=symbol]{5}{\radian\per\second}) are used. For each sample, both in training and in testing, we apply random inputs generated from a uniform distribution given by $\mathcal{U}_{u}(\SI[per-mode=symbol]{-5}{\radian\per\second\squared}, \SI[per-mode=symbol]{5}{\radian\per\second\squared})$. The observables include the angle $\theta$ and its first three derivatives, derived analytically based on the dynamics equation. 

The obtained structure of the data-driven Koopman operator resembles the Taylor-series structure \eqref{eq:: Taylor_matrix} (see Fig.~\ref{fig::DatadrivenKoopman_deviation}), which adds validity to the data-driven error bound estimation. The error bound estimates and the actual errors are shown in Fig.~\ref{fig:: ErrorBounds}. The error bound that is estimated from the structure of the data-driven operator without knowledge of the dynamics is reasonably accurate at predicting the maximum error. In addition, note that the maximum actual error and the error bound estimates have similar slopes with respect to the prediction horizon as well as the fact that both the actual error and the error bound estimates decrease with increasing order of derivatives used as basis functions. 

Next, we demonstrate the performance of the derivative-based data-driven Koopman operator and the error bound estimates when dynamics are unknown. Using the single pendulum system, only the angle and angular velocity are measured; higher-order derivatives are estimated using central finite differences \cite{olver2014introduction}. To illustrate the robustness of the approach to noise, we add zero-mean, Gaussian-distributed noise $\mathcal{N} (0, \sigma^2)$ to the measurements of $\theta$ and $\omega$, which are then also filtered through a moving average of 15 periods for noise reduction. The higher-order derivatives and the Koopman operator are computed from the filtered measurements. The term $|f^{(n+1)}_{max}|$ is computed as the maximum magnitude of $|f^{(n+1)}|$, which is also calculated using central finite differences \cite{olver2014introduction}.

In Fig. \ref{fig:: errorBounds_UnknownDynamics} we show results for $n=2$ and two levels of noise: low ($\sigma = \pi/180$) and high noise ($\sigma = 15\pi/180$). Note that the actual maximum error induced by the Koopman operator remains almost identical, indicating that the presented error bounds can be used even with high levels of noise after using a simple denoising method. The error bound estimate for the low-noise scenario follows closely the error induced by the Koopman operator. In the high-noise scenario, the error bound estimate is more conservative. This is because the error bound formula is highly dependent on the calculated term $|f^{(n+1)}_{max}|$, which is likely to be miscalculated with noisy measurements. Note that if the training data do not represent the entire state space, the calculated value for $|f^{(n+1)}_{max}|$ is likely to underestimate the true value. On the other hand, including some safety margin in the $|f^{(n+1)}_{max}|$ term can make the error bounds more conservative. These results suggest that, although the error bound estimates may become less accurate with increasing levels of noise when dynamics are unknown, simple denoising methods can render the performance of the derivative-based Koopman operator robust to noise, suggesting that the proposed methodology is a promising candidate for the prediction of unknown systems. 
 
\begin{figure}
	\centering
	\includegraphics[width=0.99\linewidth, keepaspectratio = 1]{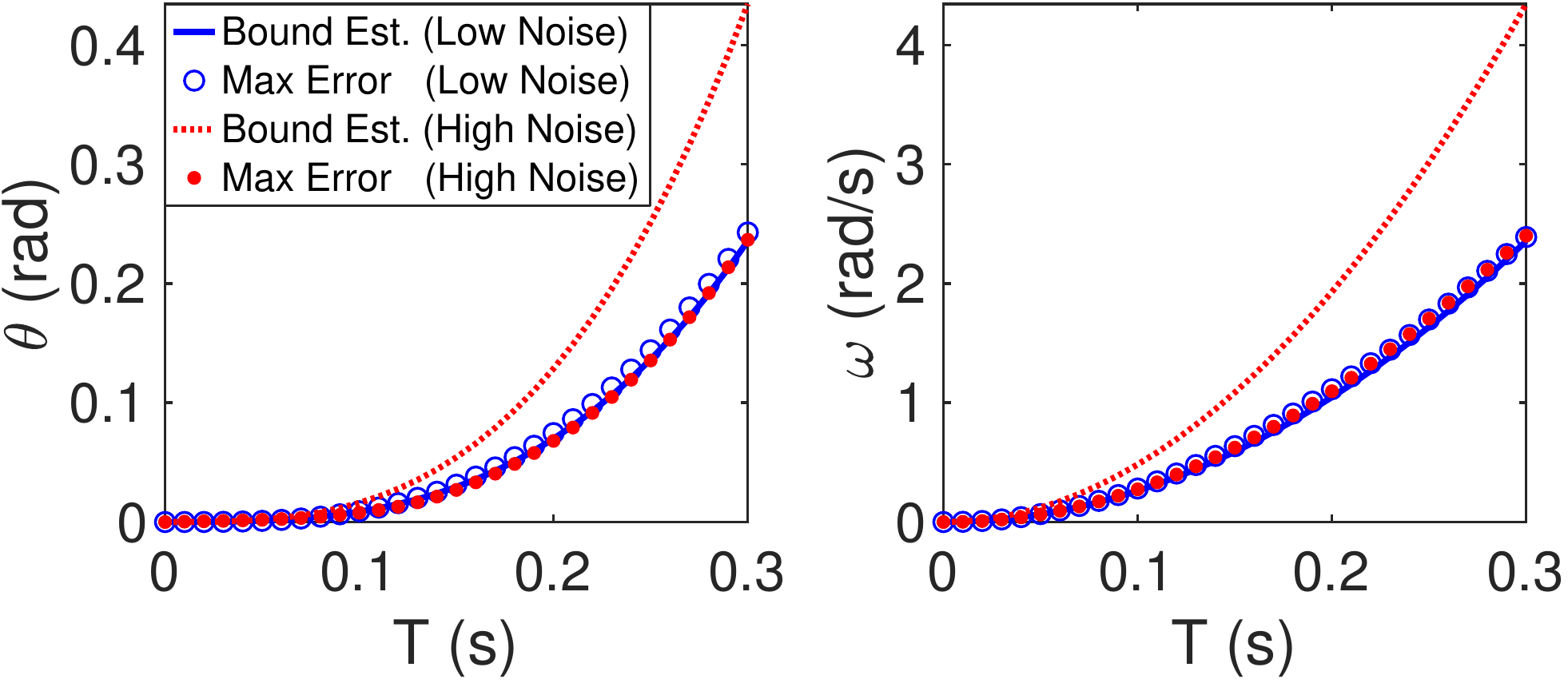}
	\caption{Simulated error bound estimates and actual maximum errors induced by the data-driven Koopman operators for the single pendulum system when dynamics are unknown and measurements are noisy. The derivative basis functions are calculated numerically from the state measurements---no analytical model is used.}\label{fig:: errorBounds_UnknownDynamics}
\end{figure}

\begin{figure*}
	\centering
 	\includegraphics[width=0.99\linewidth, keepaspectratio = 1]{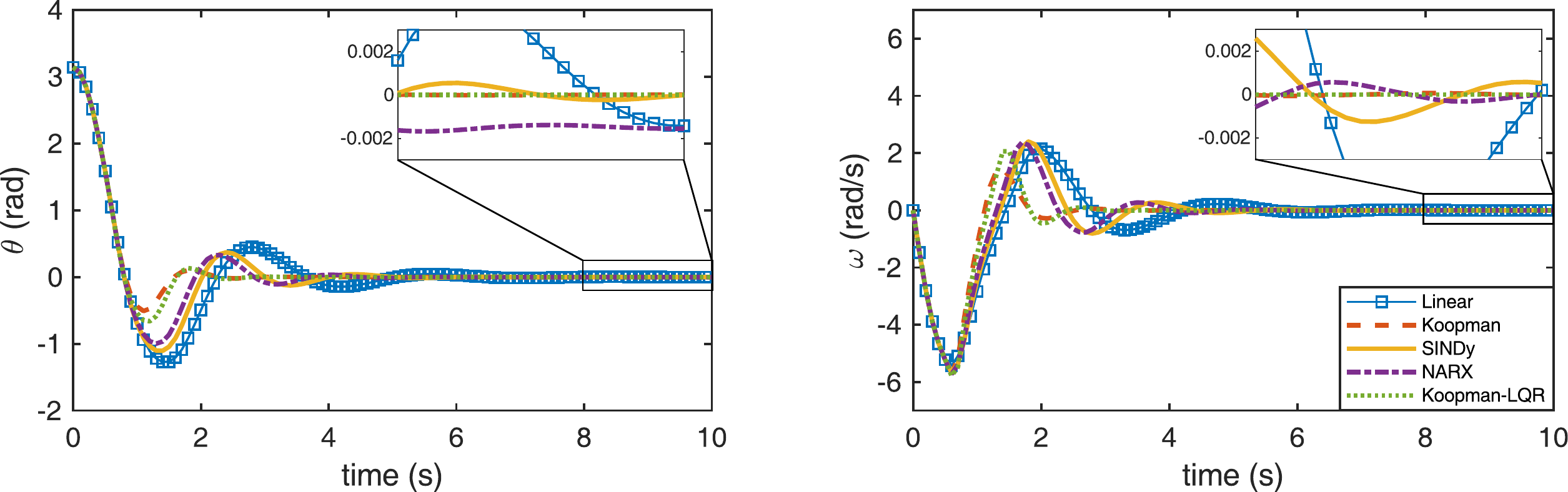} 
	\caption{Control of a pendulum system based on data-driven models obtained using SINDy, NARX, a linear model based only on the system states, and a derivative-based Koopman model whose observables contain the state $\theta$ and its first- and second-order derivatives. The derivative basis functions are numerically estimated from state measurements both for training the Koopman model and online to implement control---no analytical model is used. All models are used to design MPC control and, in addition, we use the Koopman model for LQR feedback (Koopman-LQR). Koopman with MPC has the lowest cost ($19.71$) for the 10~s simulation, while NARX, SINDy and the linear model result in errors that are $12.43\%$, $18.32\%$, and $30.61\%$ higher, respectively. Koopman-LQR leads to the second best performance ($0.97\%$ higher error in comparison to Koopman-MPC).}\label{fig:: MPC_Comparison}
\end{figure*}

In conclusion, the error bounds are derived with respect to the analytical Taylor-based Koopman form of \eqref{eq:: Taylor_matrix}. However, we show (Fig. \ref{fig::DatadrivenKoopman_deviation}) that, in practice, the data-driven Koopman model has a very similar structure to \eqref{eq:: Taylor_matrix} such that the error bounds remain relevant estimates. In addition, we verify that the error bounds reflect reasonably well the prediction error induced by data-driven Koopman models (Fig. \ref{fig:: ErrorBounds}), even in the presence of noise (Fig. \ref{fig:: errorBounds_UnknownDynamics}). In practice, the error bounds, which depend on the prediction time horizon and the magnitude of the derivatives of the dynamics, provide a systematic method to determine the basis functions for a desired balance between model accuracy and complexity.

Note that, in this work, we do not argue that one should always use additional basis functions than the original system states; instead, one could analyze when and how to augment the basis functions of Koopman operators to improve the modeling accuracy while being cognizant of increased system order and complexity. In particular, the derived error bound estimates can serve as a guide on whether one should do so and by how many derivatives. 
\subsection{Comparison to alternative system identification methods}
\par Finally, we compare the performance of the derivative-based data-driven Koopman representation approach to alternative system identification methods using the single pendulum system. Specifically, we use SINDy \cite{SINDY}, NARX \cite{NSID}, a data-driven linear model, and the proposed derivative-based data-driven Koopman approach to obtain models and design MPC control to invert the pendulum to the upright position.

To train the models we generate trajectories by forward simulating 500 uniformly-sampled initial states $s_0$ for 0.04 seconds using time steps of $\Delta t=0.01$~s. That is, each of the 500 training trajectories contains four measurements, which are used to compute derivatives for the middle measurements needed for the SINDy algorithm and the proposed method, as well as make use of delays for NARX. The initial states are sampled from uniform distributions given by $\mathcal{U}_{\theta_0}$($-2\pi$\,\si{\radian}, $2\pi$\,\si{\radian}) and $\mathcal{U}_{\omega_0}$($-$\SI[per-mode=symbol]{5}{\radian\per\second}, \SI[per-mode=symbol]{5}{\radian\per\second}). We use random inputs generated from a distribution given by $\mathcal{U}_{u}(\SI[per-mode=symbol]{-10}{\radian\per\second\squared}, \SI[per-mode=symbol]{10}{\radian\per\second\squared})$. We train the SINDy model using the open-source package \cite{pysindy} and the NARX model using the MATLAB Deep Learning Toolbox \cite{MatlabOTB}. The higher-order derivatives used for SINDy and the derivative-based Koopman model are numerically estimated from the angle and angular velocity measurements.

Fig.~\ref{fig:: MPC_Comparison} shows simulation results for the inversion of the single pendulum system using MPC control computed from models obtained with NARX, SINDy, a data-driven linear model based only on the system states ($n=1$, see Fig. \ref{fig:: ErrorBounds}), and a data-driven Koopman representation based on the state $\theta$ and its first- and second-order derivative ($n=2$). In addition, we also test LQR control on the Koopman model to demonstrate that similar control performance can be achieved with LQR gains that are computed once. The desired states are given by $s_{des} = [0,0]$, the weights for the states and control are $Q = \text{diag}(5,0.01)$ and $R = 0.001$, respectively, and the prediction horizon is $T=0.1$~s. 

We also compare the control performance of these methods for 30 initial conditions $\theta_0$ and $\omega_0$ sampled from uniform distributions given by $\mathcal{U}_{\theta_0}$($-\pi$\,\si{\radian}, $\pi$\,\si{\radian}) and $\mathcal{U}_{\omega_0}$($-$\SI[per-mode=symbol]{2}{\radian\per\second}, \SI[per-mode=symbol]{2}{\radian\per\second}). The average of the 30 errors is 6.00 (with a standard deviation of 6.60) for the Koopman-LQR approach; 6.05 (with a standard deviation of 6.59) for the Koopman-MPC; 6.48 (with a standard deviation of 7.00) for NARX; 6.69 (with a standard deviation of 7.18) for SINDy; and 7.44 (with a standard deviation of 7.70) for the linear model. 

These results show that the control performance of the Koopman-MPC method is marginally better than the linear-MPC, NARX-MPC and SINDy-MPC. Note that in simulation the comparative performance  of the methods considered is similar using prediction horizons up to $T=1$~s. Given that NARX is computationally expensive, here we report the results for a short planning horizon so that the performance is closer to real-time execution. Also, the Koopman-LQR approach delivers control performance comparable to the Koopman-MPC method, with the additional benefit that it lends itself to efficient control computation, which makes it an attractive choice for online robotic control applications. This is why we use Koopman-LQR in the simulations and experiments with the robotic fish in Section \ref{sec::Results}.

\section{Data-Driven Control of Tail-Actuated Robotic Fish}\label{sec::Results}

\begin{figure*}
	\centering
	\includegraphics[width=1\linewidth, height = 0.30\textheight, keepaspectratio= true]{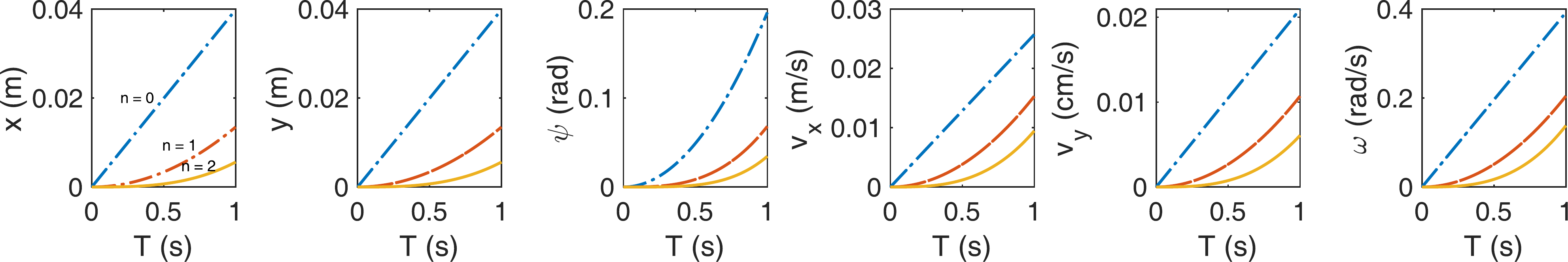}
	\caption{Error bound estimates based on derivative-based Koopman models of the robotic fish dynamics \eqref{eq:: dynamics} for increasing order of derivatives used in the basis functions. The derivative basis functions are constructed analytically from the known dynamics. Each additional order of derivatives improves the error bound estimates over the selected prediction horizon. The error bound estimates are computed using \eqref{eq:: GlobalError_simplified} where $|f_{max}^{(n+1)}|$ is calculated from the training data.}\label{fig::ErrorBoundsFish}
\end{figure*}

We illustrate and validate the proposed data-driven modeling approach using a tail-actuated robotic fish. The states of the robotic fish are $s = [x, y, \psi, v_x, v_y, \omega]^T$, where $x, y$ are the 2D world-frame coordinates, $\psi$ is the orientation, $v_x$ and $v_y$ are the body-frame linear velocities (surge and sway, respectively), and $\omega$ is the body-frame angular velocity. We use $\alpha$ to indicate the angle of the tail, actuated with $\alpha(t) = \alpha_o + \alpha_a \sin(\omega_a t)$, where $\alpha_a, \alpha_o, \omega_a$ are the amplitude, bias, and frequency of the tail beat. To simplify the problem, we keep the frequency fixed at $\omega_a = 2\pi$~$\si{\radian/\second}$.

We describe the dynamics of the system with an average model \cite{averagemodel} given by
\begin{align}\label{eq:: dynamics}
\dot{s} =
\begin{bmatrix}
 \dot{x} \\ \dot{z} \\ \dot{\psi} \\ \dot{v}_x \\ \dot{v}_y \\ \dot{\omega}
\end{bmatrix}
\stackrel{\bigtriangleup}{=}
\begin{bmatrix}
v_x \cos(\psi) - v_y \sin(\psi)\\
v_x \sin(\psi) + v_y \cos(\psi)\\
\omega\\
f_1(s) + K_f f_4(\alpha_0, \alpha_a, \omega_a)\\
f_2(s) + K_f f_5(\alpha_0, \alpha_a, \omega_a)\\
f_3(s) + K_m f_6(\alpha_0, \alpha_a, \omega_a)
\end{bmatrix},
\end{align}	
where
\begin{equation}
\resizebox{0.99\hsize}{!}{$
\begin{aligned}
&f_1(s)=~\frac{m_2}{m_1} v_y \omega - \frac{c_1}{m_1}v_x\sqrt{v_x^2 + v_y^2} + \frac{c_2}{m_1} v_y\sqrt{v_x^2 + v_y^2} \arctan(\frac{v_y}{v_x}) 
\\&f_2(s)=~-\frac{m_1}{m_2}v_x \omega - \frac{c_1}{m_2}v_y\sqrt{v_x^2 + v_y^2} - \frac{c_2}{m_2}v_x\sqrt{v_x^2 + v_y^2}\arctan(\frac{v_y}{v_x})
\\&f_3(s)=~(m_1-m_2) v_xv_y - c_4 \sgn(\omega) \omega^2 
\\&f_4(\alpha_0, \alpha_a, \omega_a)=~\frac{m}{12m_1} L^2\omega_a^2 \alpha_a^2 (3 - \frac{3}{2}\alpha_o^2 - \frac{3}{8}\alpha_a^2)
\\&f_5(\alpha_0, \alpha_a, \omega_a)=~\frac{m}{4m_2} L^2 \omega_a^2 \alpha_a^2\alpha_o
\\&f_6(\alpha_0, \alpha_a, \omega_a) =~-\frac{m}{4J_3} L^2c \omega_a^2 \alpha_a^2\alpha_o
\end{aligned}
$}
\end{equation}
and $m_1 = m_b - m_{a_x}, m_2 = m_b - m_{a_y}, J_3 = J_{b_z} - J_{a_z}, c_1 = \frac{1}{2}\rho SC_D, c_2 = \frac{1}{2}\rho SC_L, c_4 = \frac{1}{J_3}K_D, c_5 = \frac{1}{2J_3}L^2mc$. Parameter $m_b$ is the mass of the robotic fish, $m_{a_x}$ and $m_{a_y}$ are the hydrodynamic derivatives that represent the added masses of the robotic fish along the $x$ and $y$ directions, respectively, $J_{a_z}$ and $J_{bz}$ are the added inertia effect and the inertia of the body about the $z$-axis, respectively, $m$ is the mass of the water displaced by the tail per unit length, $\rho$ is the water density, $L$ is the tail length, $c$ is the distance from the body center to the pivot point of the actuated tail, $C_D, C_L, K_D$ are drag force, lift, and drag moment coefficients, respectively, and $K_f$ and $K_m$ are scaling coefficients measured experimentally\cite{NMPC_maria}. 

We train a Koopman operator using the control-affine form of the dynamics \eqref{eq:: dynamics} that is obtained by substituting 
\begin{align}
&u_1 =~\alpha_a^2 (3 - \frac{3}{2}\alpha_o^2 - \frac{3}{8}\alpha_a^2)
&u_2 =~\alpha_a^2\alpha_o.
\end{align}
Because the computed control is in terms of the variables $u_1$ and $u_2$, it needs to be mapped to implementable values for the amplitude, $\alpha_a$, and the bias, $\alpha_o$, of the tail flapping. To convert $u_1$ and $u_2$ back to the physical actuation variables, we use a constrained global minimization solver based on Sequential Quadratic Programming (SQP) that finds the nearest, in the space of $u_1$ and $u_2$, feasible actuation values for $\alpha_a$ and $\alpha_o$. Given $u_1$ and $u_2$, the constrained optimization problem is posed as
\begin{equation}
\begin{aligned}
 \underset{\alpha_a, \alpha_o}{\operatorname{argmin}} \sqrt{\big(u_1 - \alpha_a^2 (3 - \frac{3}{2}\alpha_o^2 - \frac{3}{8}\alpha_a^2)\big)^2 + \big(u_2 - \alpha_a^2\alpha_o\big)^2} \\
 \text{subject to:~} \alpha_a \in [0, 30^{\circ}] ~ \text{AND~} \alpha_o \in [-45^{\circ}, 45^{\circ}].
 \end{aligned}
 \end{equation}
This minimization is solved before every control update. 

Given the unilateral constraints on the forward motion of the tail-actuated robotic fish (it cannot move backwards), directly tracking position coordinates becomes rather challenging. For example, when the target lies behind the robotic fish, the control solution generates a negative amplitude (to generate backward motion) that is infeasible and thus the system stops moving. This behavior has been observed and tackled in \cite{BSC_maria} by translating position coordinates into different error states, associated with the body-frame velocities of the system. Similarly, in this work, we argue that one can express all feasible trajectories for the tail-actuated robotic fish in terms of an angle and a forward velocity profile. 

\subsection{Simulation Results}
In this section, we present simulation results on the data-driven modeling and LQR control on the tail-actuated robotic fish. To decide the optimal order of derivative basis functions, we compare the error bound estimates over $T = 1$~s, which is the feedback rate, using different orders of derivatives. Results are shown in Fig.~\ref{fig::ErrorBoundsFish}. We select $n=2$ for a reasonable balance between the increasing complexity of calculating higher-order derivatives and model accuracy. Next, we populate the observable functions with the states, their first- (using \eqref{eq:: dynamics}) and second-order derivatives, which are derived analytically from the average model. To allow the identification of unknown or changing coefficients (as discussed in Section \ref{SubSection:: Synthesis}), we consider each term in the derivatives individually. For example, $\frac{d}{dt} v_y \omega = \dot{v}_y \omega + v_y \dot{\omega}$ generates multiple basis functions, where $\dot{v}_y$ and $\dot{\omega}$ are given by \eqref{eq:: dynamics}. Using separate functions for the time derivatives of each individual term that appears in the dynamics is similar to using the time derivatives of the entire equation of a state (e.g. $\ddot v_y(t)$). Despite increasing the number of basis functions, we prefer the first approach because it does not require knowing the coefficients of the individual terms in advance (e.g. $\frac{m_2}{m_1}$). As a result, we can readily train the Koopman operator on other robotic tail-actuated fish that have a different morphology. In this way, we end up with the system states, the control inputs, and 54 additional scalar functions, with $\Psi_s(s) \in \mathbb{R}^{60}$ and $\Psi_u(u) = u \in \mathbb{R}^{2}$.

Note that we choose control-dependent basis functions to be $u$ so as to use LQR feedback. One can also choose basis functions that include nonlinear control terms in combination with a different control policy, such as NMPC \cite{koopman_kronic, bruderBilinearization}. Alternatively, one can always convert dynamics that are nonlinear in control by introducing new dummy variables (e.g. $v_i$) as the control input that are the derivatives of the system actuation ($\dot{u}_i = c_i (v_i - u_i))$, where $c_i \in \mathbb{R}^+$ dictate the rate of change \cite{koopman_kronic}. This is in practice also closer to the physical implementation of actuation that cannot instantly change values. In this way, it is then possible to include nonlinearities in $u_i$, while the system remains still linear with respect to the control input $v_i$ designed by the user.
\begin{table}
\caption{Simulation parameters for the tail-actuated fish model dynamics \eqref{eq:: dynamics}.}
\begin{center}	\sisetup{table-format = 1.3e2, table-number-alignment = center}
\begin{tabular}{|l|l|l|l|}
	\hline
	\multicolumn{4}{|c|}{Simulation Parameters} \\
	\hline
	 Parameter& Value &Parameter& Value\\
	\hline	
	$m_b$ &\hphantom{$-$}\SI{0.725}{\kilogram} & $\rho$ &\SI{1000}{\kilogram\per\meter\cubed} \\
$m_{ax}$ &\SI{-0.217}{\kilogram} & $S$ & \SI{0.03}{\meter\squared} \\
$m_{ay}$ &\SI{-0.7888}{\kilogram} & $C_D$ & \num{0.97} \\
$J_{bz}$ &\hphantom{$-$}\SI{2.66d-3}{\kilogram\meter\squared} &$C_L$ & \num{3.9047} \\
$J_{az}$ &\SI{-7.93d-4}{\kilogram\meter\squared}& $K_D$ &\num{4.5d-3} \\
$L$ & \hphantom{$-$}\SI{0.071}{\meter} & $K_f$ & \num{0.7} \\
$d$ & \hphantom{$-$}\SI{0.04}{\meter} & $K_m$ & \num{0.45} \\
$c$ & \hphantom{$-$}\SI{0.105}{\meter}&~ &~ \\
\hline
\end{tabular}
\end{center}
 \label{table:: sim_parameters}
\end{table}

Next, we train an approximate Koopman operator using \eqref{eq:: Kd_AG}. To generate data $s_k$ and $s_{k+1}$, we sample $P = 3000$ initial conditions for the states with uniform distributions given by $\mathcal{U}_{x_0}$(\SI{-0.5}{\meter}, \SI{0.5}{\meter}),
$\mathcal{U}_{y_0}( $$-$\SI{0.1}{m}, \SI{0.1}{m}), $\mathcal{U}_{\psi_0}$($-\pi/4$\,\si{\radian}, $\pi/4$\,\si{\radian}), $\mathcal{U}_{v_x}$(0, \SI[per-mode=symbol]{0.04}{\meter\per\second}), $\mathcal{U}_{v_y}$($-$\SI[per-mode=symbol]{0.0025}{\meter\per\second}, \SI[per-mode=symbol]{0.0025}{\meter\per\second}), $\mathcal{U}_{\omega}$($-$\SI[per-mode=symbol]{0.5}{\radian\per\second}, \SI[per-mode=symbol]{0.5}{\radian\per\second}). For each sample, we apply random inputs generated from a uniform distribution given by $\mathcal{U}_{\alpha_0}$(\ang{-45}, \ang{45}) for the tail angle bias and $\mathcal{U}_{\alpha_a}(0, 30^{\circ})$ for the tail angle amplitude of oscillations. Then, for each sample of initial conditions $s_{k}$ and controls $u_k$, we use dynamics \eqref{eq:: dynamics} and parameters shown in Table \ref{table:: sim_parameters} to propagate the states with the given control for $\Delta t = \SI{0.005}{\second}$ and obtain the final states $s_{k+1}$. We use the set of $s_k, s_{k+1}, u_k$ to compute the approximate discrete Koopman operator \eqref{eq:: Kd_AG}. Note that the value of $u_{k+1}$ can be arbitrary, since we are not trying to predict the evolution of the control-dependent basis functions. 
Once we have trained the Koopman operator, we convert it to the continuous time via $\tilde{\mathcal{K}} = \log(\tilde{\mathcal{K}}_d)/\Delta t$, extract the state- and control-linearization matrices $A$ and $B$, choose the weight matrices $Q$ and $R$ and compute the infinite-horizon LQR gains.

\begin{figure}
	\centering
	\includegraphics[width=1\linewidth, height = 0.30\textheight, keepaspectratio= true]{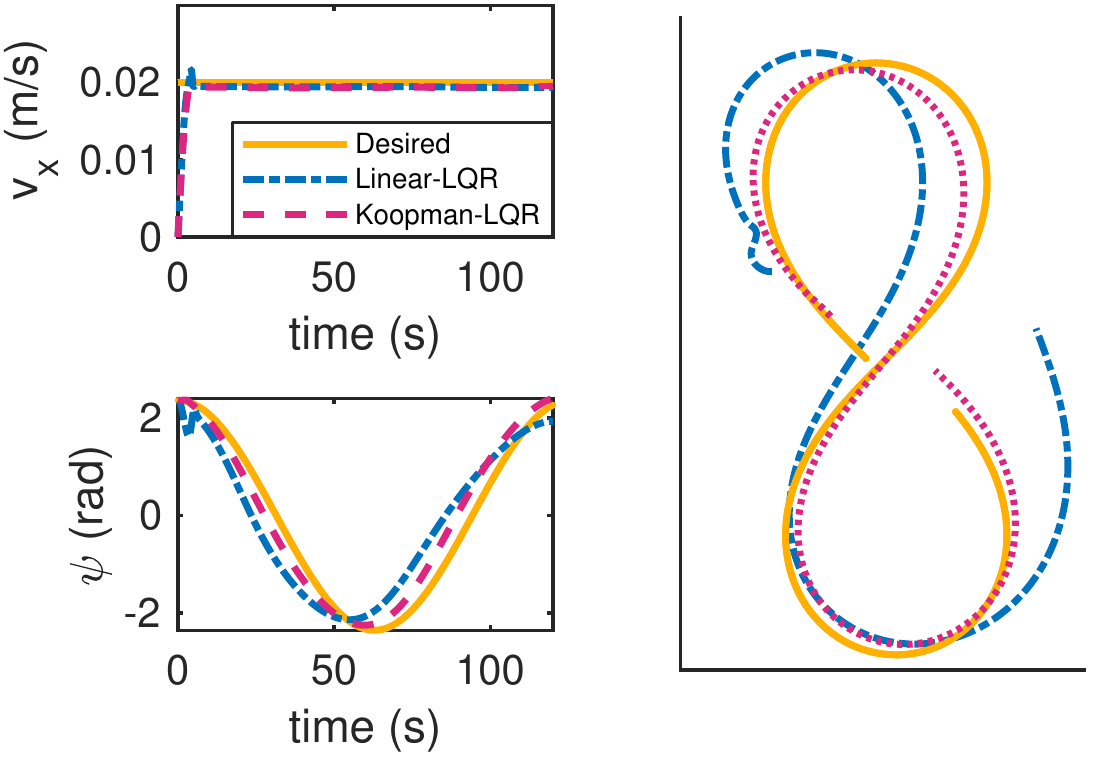}
	\caption{LQR-controlled robotic fish in simulation. The LQR gains are generated once using the learned Koopman operator. The derivative basis functions are constructed analytically from the known dynamics. The desired trajectory is given in terms of the angle and the forward velocity. Since the position coordinates are not included in the performance objective \eqref{eq:: J_K}, the controlled trajectories are individually shifted to align with the desired figure-8 shape as closely as possible. Despite using fixed LQR gains, the controlled systems successfully track the desired states that were designed to produce a figure-8 pattern. Koopman-LQR has the lower cost ($3.35$) for the $120$ s duration, while the approach based on the data-driven linear model with the same set of states as in \eqref{eq:: dynamics} results in an error that is $95\%$ higher.}\label{fig::Sim_Figure8}
\end{figure}

\figurename~\ref{fig::Sim_Figure8} shows the velocity tracking performance of the derivative-based Koopman model in comparison to a linear data-driven model for the system \eqref{eq:: dynamics} (with the same set of states as in \eqref{eq:: dynamics}) when using LQR-feedback. The desired trajectory is a figure-8 described by $s_{des} = [0,0,135 \cdot \pi/180 \cdot \sin(0.05t + \pi/2),0.02, 0, 0.05 \cdot 135 \cdot \pi/180 \cdot \cos(0.05t + \pi/2)]$. The weights are $Q = \text{diag}(0,0,0.1, 4000, 0, 0)$ and $R = \text{diag}(0.01, 0.01)$. As is seen in the figure the proposed derivative-based Koopman modeling approach leads to improved control performance.

\subsection{Experimental Results}
{
\smallskip
\setlength\extrarowheight{3pt}
\begin{table*}
	\caption{Amplitude and bias inputs used to collect training dataset.}\label{table:: actuation}

	\centering
	\begin{adjustbox}{width=2\columnwidth,center}
	\begin{tabular}{|c|cccccccccccccccccccc}
		\cline{1-21}
	 \multicolumn{1}{|c|}{} & \multicolumn{20}{c|}{Actuation values} \\ \cline{1-21} 
						\multicolumn{1}{|c|}{Amp ($^{\circ}$)} & \multicolumn{5}{c|}{15}& \multicolumn{5}{c|}{20} & \multicolumn{5}{c|}{25} & \multicolumn{5}{c|}{30} \\ \cline{1-21} 
						
						\multicolumn{1}{|c|}{Bias ($^{\circ}$)} 
						& \multicolumn{1}{c|}{0}& \multicolumn{1}{c|}{$\pm 20$} & \multicolumn{1}{c|}{$\pm 30$} & \multicolumn{1}{c|}{$\pm 40$} & \multicolumn{1}{c|}{$\pm 50$} & \multicolumn{1}{c|}{0}& \multicolumn{1}{c|}{$\pm 20$} & \multicolumn{1}{c|}{$\pm 30$} & \multicolumn{1}{c|}{$\pm 40$} & \multicolumn{1}{c|}{$\pm 50$} & \multicolumn{1}{c|}{0} & \multicolumn{1}{c|}{$\pm 20$} & \multicolumn{1}{c|}{$\pm 30$} & \multicolumn{1}{c|}{$\pm 40$} & \multicolumn{1}{c|}{$\pm 50$} & \multicolumn{1}{c|}{0} & \multicolumn{1}{c|}{$\pm 20$} & \multicolumn{1}{c|}{$\pm 30$} & \multicolumn{1}{c|}{$\pm 40$} & \multicolumn{1}{c|}{$\pm 50$} \\ \cline{1-21}
	\end{tabular}
		\end{adjustbox}
\end{table*}
}

\subsubsection{Experimental Setup}
\begin{figure}
	\centering
	\includegraphics[width=1\linewidth, height = 0.2\textheight]{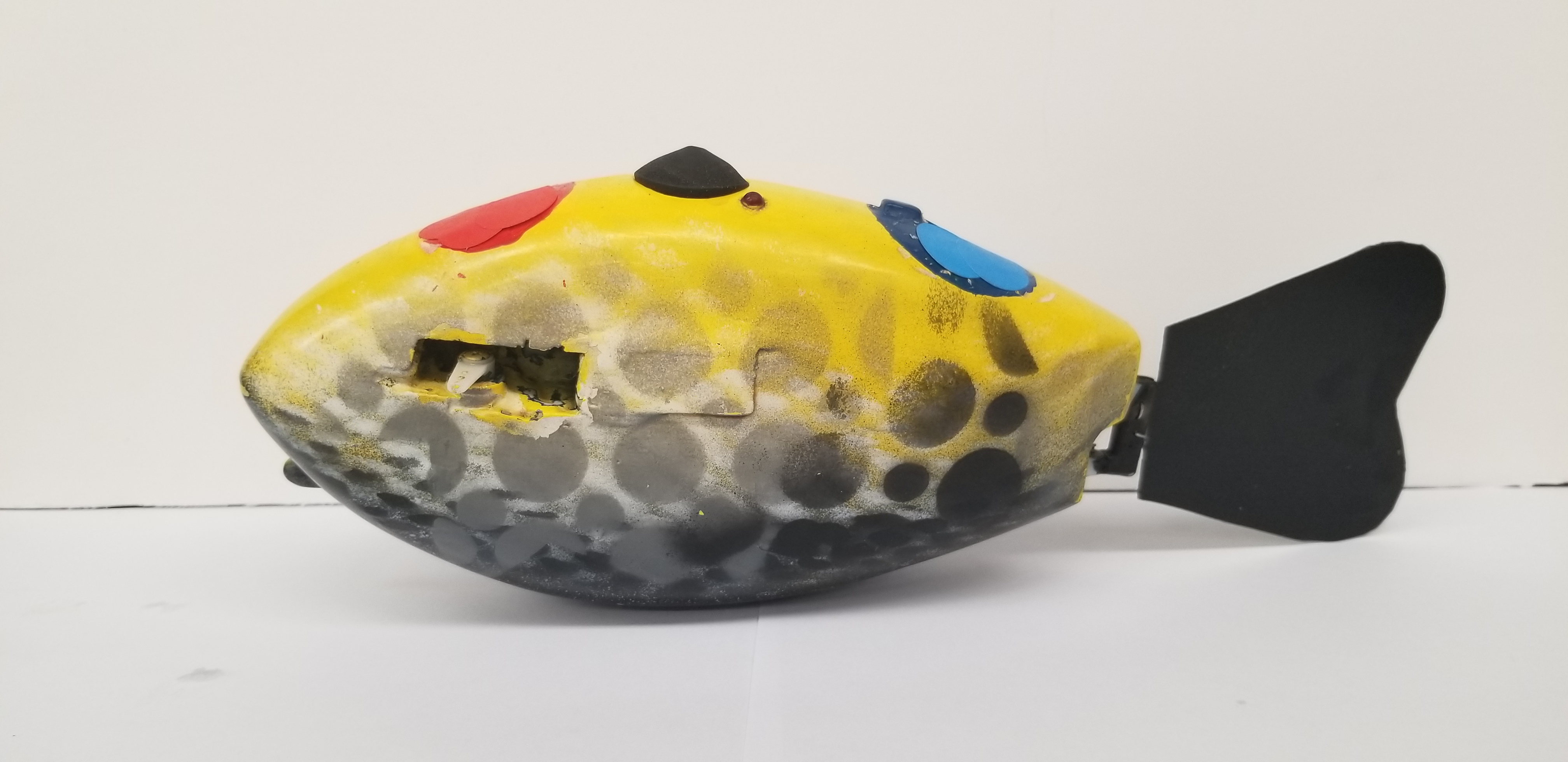}
	\caption{Tail-actuated robotic fish used in experiments, developed by the Smart Microsystems Lab at Michigan State University. It maneuvers in water by oscillating its tail fin.}\label{fig:: roboticfish}
\end{figure}
\begin{figure*} 
	\begin{subfigure}[b]{0.98\columnwidth}
		\centering
		\includegraphics[width=\linewidth,height = 0.25\textheight, keepaspectratio = true]{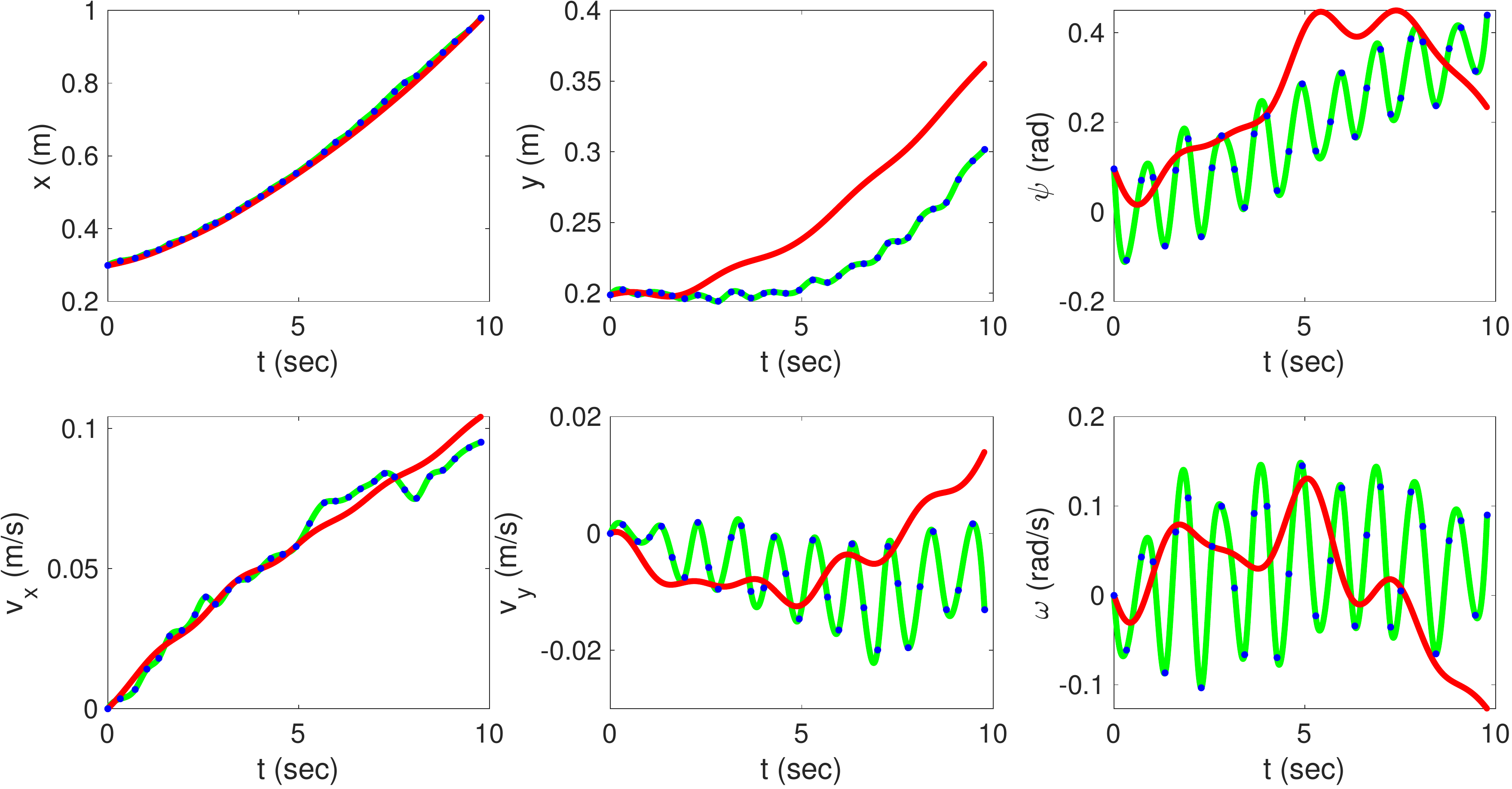} 
		\caption{$\alpha_0 = 0, \alpha_a = \ang{30}$} 
	\end{subfigure}
	\hfill 
	\begin{subfigure}[b]{0.98\columnwidth}
		\centering
		\includegraphics[width=\linewidth,height = 0.25\textheight, keepaspectratio = true]{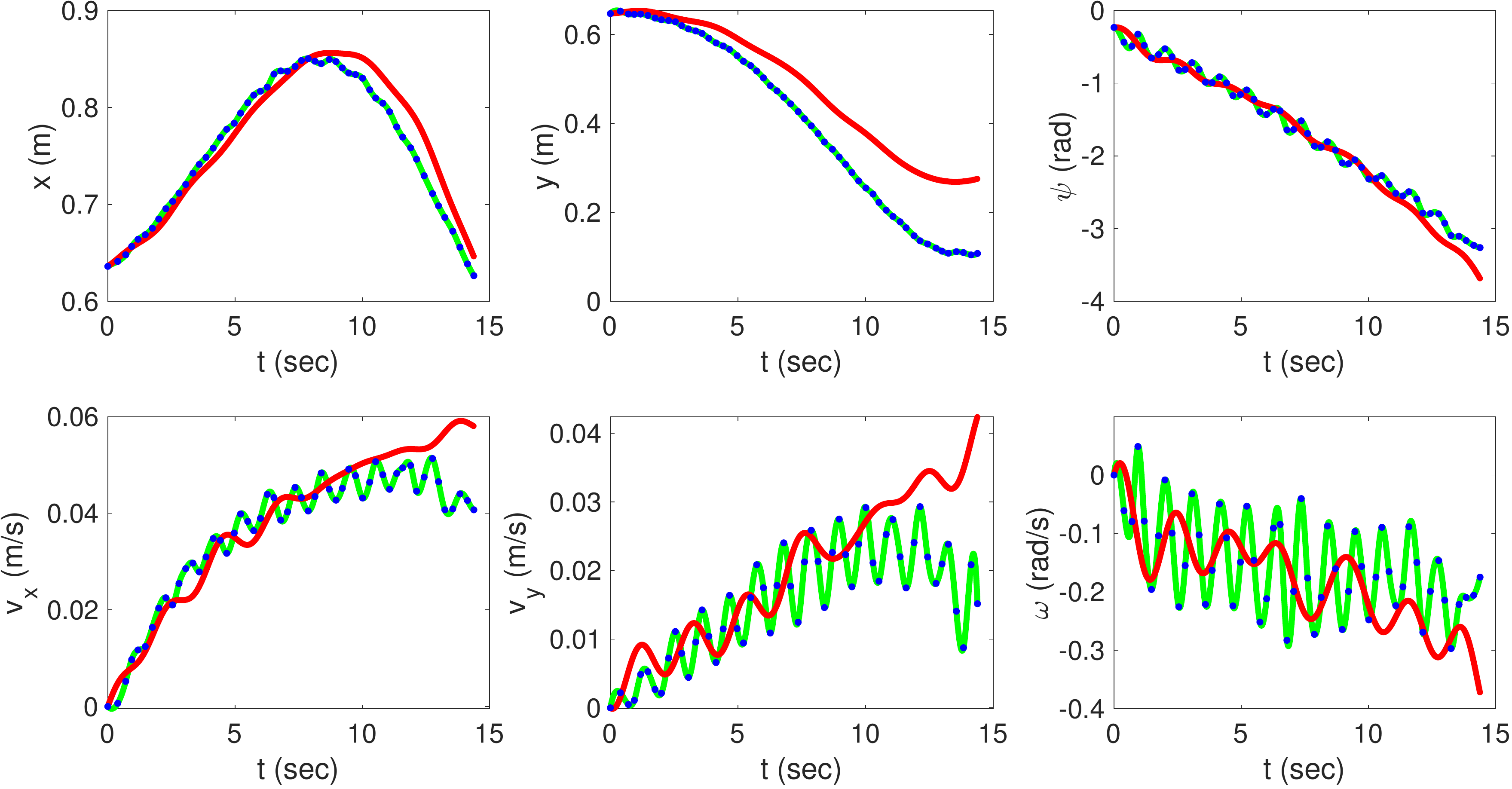} 
		\caption{$\alpha_0 = \ang{45}, \alpha_a = \ang{25}$} 
	\end{subfigure} \hfill
	\caption{Fitness between Koopman model and experimental measurements. The green line shows data interpolated from experimental measurements (blue dots) every $\Delta t = \SI{0.005}{\second}$. The red line shows the evolution of the states using the Koopman model. The actuation is constant for each of the two runs and is indicated in the caption.}
	\label{fig:: Koopman_exp_model} 
\end{figure*}
We next use the physical robot shown in Fig. \ref{fig:: roboticfish} to experimentally test our approach. An overhead camera captures the red and blue marks on the robotic fish (see Fig. \ref{fig:: roboticfish}) and calculates the coordinates of its center and its orientation at about $\SI{3}{\hertz}$. The body-frame velocities are estimated using a Kalman filter. The state derivative functions, provided by the average model, are then evaluated with the states. To simplify the modeling and control task, the tail-beat frequency used in both the training phase and the testing phase is kept constant at $\omega_a = \SI{2\numpi}{\radian\per\second}$. In order not to disturb the periodic movement of the tail oscillation during tracking, feedback control is updated at roughly one second. The control commands are communicated to the robot via Xbee (RF communication). 

\subsubsection{Training Phase}
For the training phase, we collect experimental measurements using the robotic fish to compute the approximate Koopman operator. Throughout each run, we apply constant tail bias and amplitude for the oscillations of the tail fin. We conduct a total of 72 runs, with two trials for each of the 36 different combinations of actuation parameters, shown in Table \ref{table:: actuation}. We train the Koopman operator using the same basis functions as discussed in the simulation section. 

To create a consistent mapping with the Koopman operator, all pairs of measurements $s_k$ and $s_{k+1}$ need to be spaced (in time) equally apart, as we explain in Section \ref{sec:: Koopman}. For this reason, and also to decrease the time between measurements to fine levels (without the constraints of our sampling and filtering methods), the obtained data is interpolated at $\Delta t = \SI{0.005}{\second}$. The interpolated data is then used to obtain an approximate Koopman operator according to \eqref{eq:: Koopman_LS_solution}.

To measure how well the Koopman model captures the nonlinear dynamics of the robotic fish, we use $\tilde{\mathcal{K}}$, learned from the experimental data, to propagate the identified model continuously based on the initial states of each of the 72 experimental runs. Then, the predicted simulated trajectories are compared against the corresponding experimental ones. For the purposes of illustration, two such comparisons are shown in Fig. \ref{fig:: Koopman_exp_model}. The linear Koopman model, despite not perfect, reasonably follows the experimental data for at least five seconds. Note that, because we only minimize the single-step prediction error \eqref{eq:: Kd_AG}, it is more likely (compared to minimizing a multi-step error) that we compute an unstable Koopman operator, even if the dynamics are stable. In that case, the long-term predictions would exponentially deviate and become inaccurate. Imposing stability properties on the operator is the focus of ongoing work \cite{Mamakoukas_stableLDS, Mamakoukas_stableKoopman}.

\subsubsection{Testing Phase}
We use the data-trained Koopman operator to implement linear feedback control (LQR) for tracking. Using the weight matrices $Q$ and $R$, which penalize the tracking error and control effort, respectively, we define the minimization problem \eqref{eq:: J_K} and calculate the infinite-horizon LQR gains. Contrary to work in \cite{RSS2019_MamakoukasCastano}, and in order to illustrate the simplicity and robustness of the proposed scheme, we keep the same weights across all different tasks, such that the same LQR gains, (unless the model is updated) are used in every type of trajectory. The resulting feedback has the form shown in \eqref{eq:: K_LQR}. As mentioned earlier, we argue that, to follow any trajectory, it suffices to track a desired orientation and forward velocity and so we design the LQR weights accordingly. Specifically, the weights used are $Q = \text{diag}(0,0,0.1, 1000,0,0)$ and $R = \text{diag}(1,1)$. The weights for the angle and forward velocity are disproportionate to account for the difference in scale between the velocity of the robotic fish, typically in the order of $0.01$~m/s, and the body orientation, expressed in radians.

\begin{figure}
	\centering
	\includegraphics[width=1\linewidth, keepaspectratio = 1]{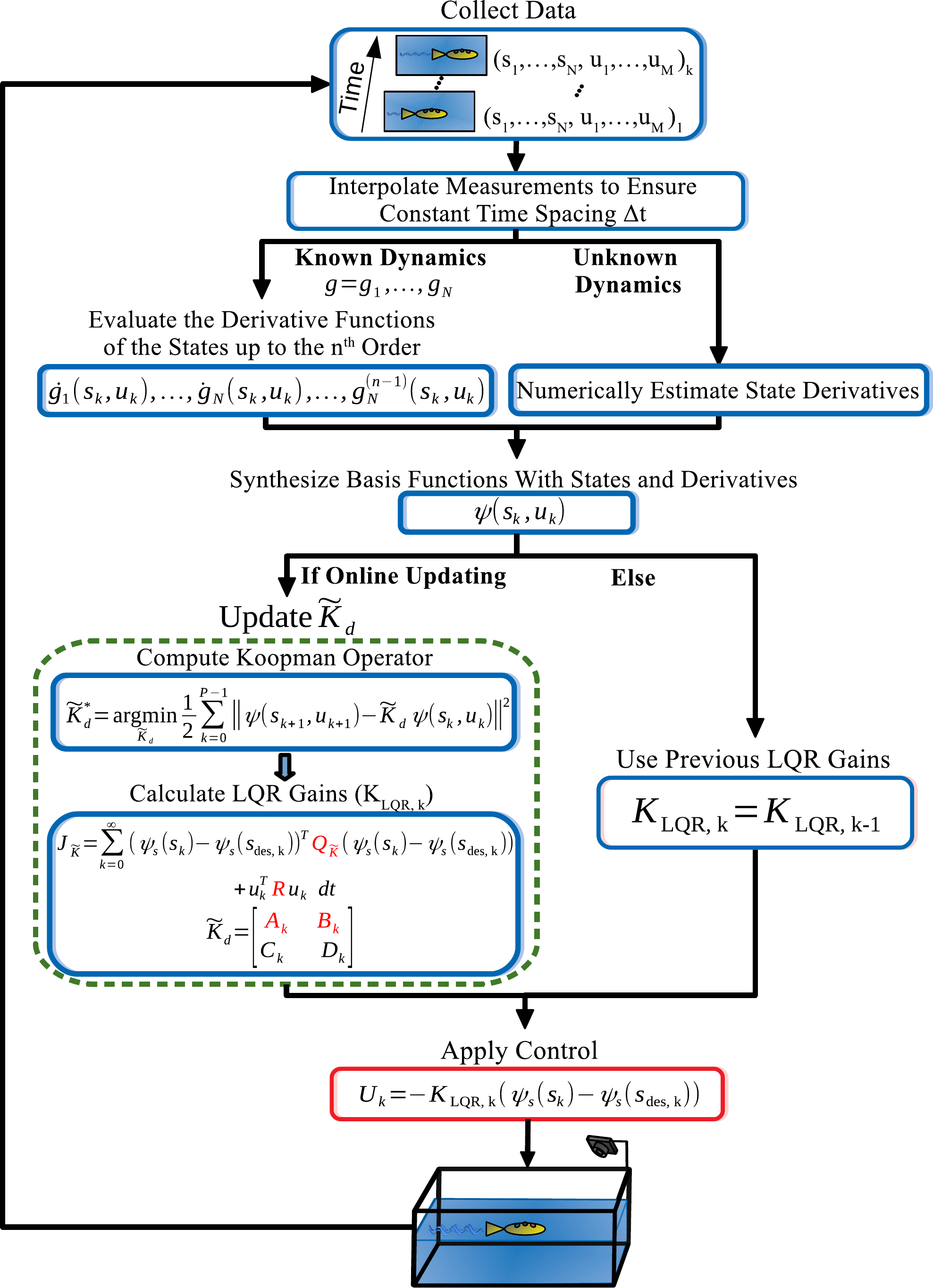}
	\caption{Outline of the proposed methodology for LQR control using derivative-based Koopman operators.}\label{fig:: methodOutline}
\end{figure}
\begin{figure}
	\centering
	\includegraphics[width=1\linewidth, height = 0.20\textheight, keepaspectratio = true]{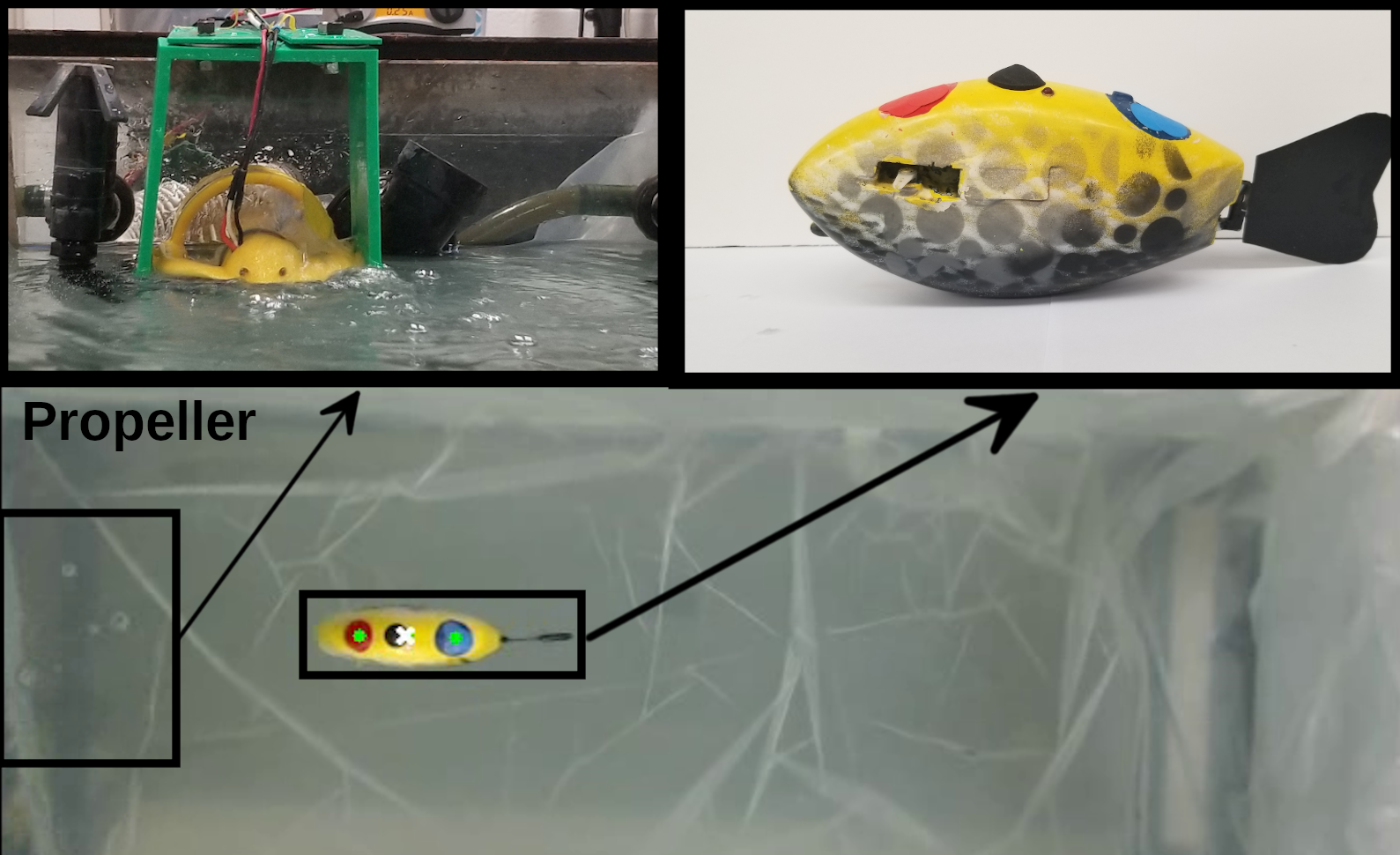}
	\caption{Experimental setup for creating fluid disturbances. The motor is halfway submerged in water, generating ripples with its propellers. }\label{fig:: propeller}
\end{figure}

\begin{figure*} 
		\centering
		\includegraphics[width=0.85\linewidth, keepaspectratio = true]{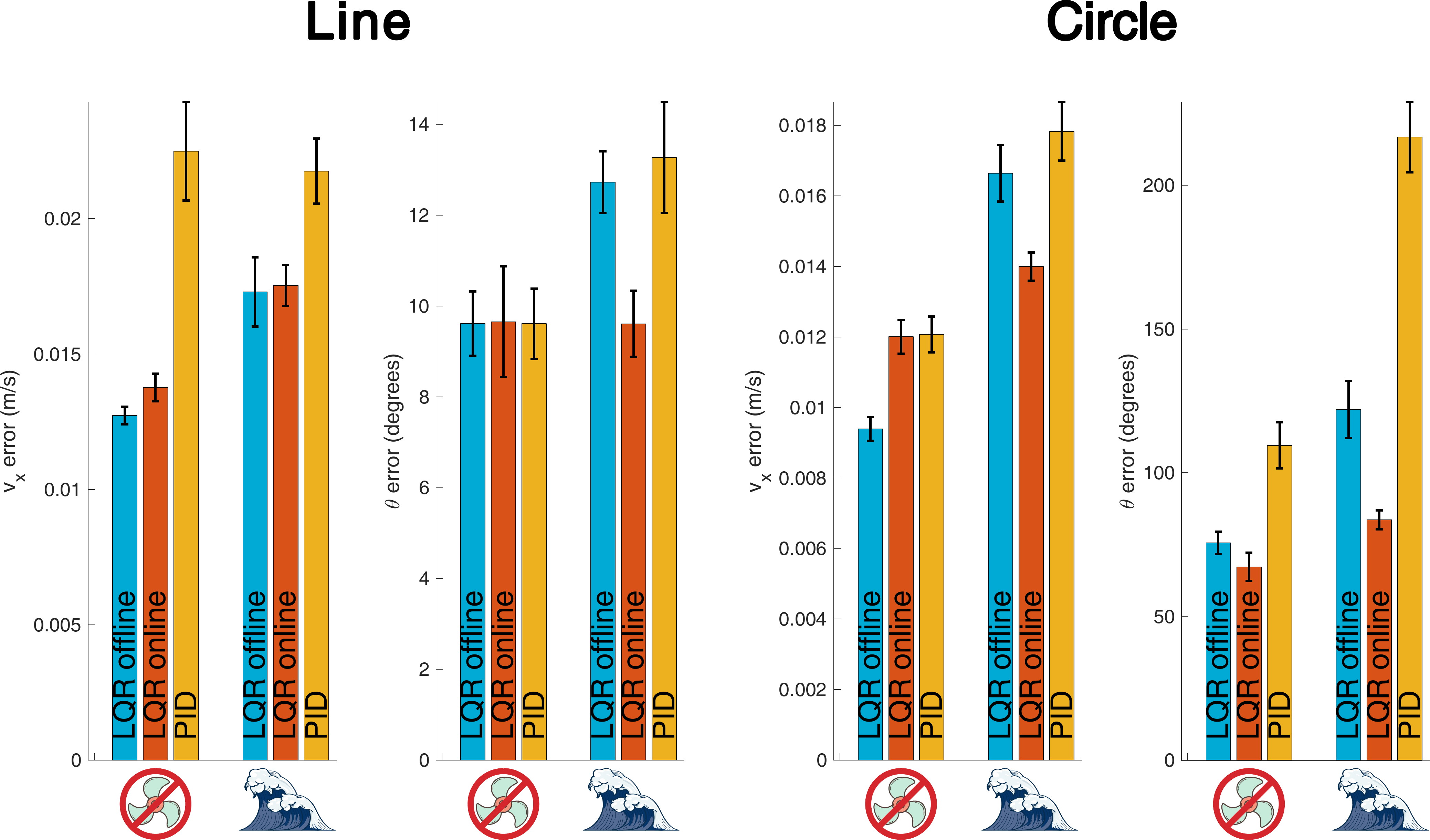} 
		\caption{Experimental results: Average error scores for velocity and angle tracking for PID and two variants of Koopman-LQR, trained offline and updated online, with and without fluid flow indicated respectively with the waves and no-fan icons. The four subplots compare the performance for the linear and circular trajectories for each state separately. The error bars indicate the standard error. The proposed Koopman operator scheme outperforms PID in all tests. Further, updating Koopman-LQR online improves the performance in the presence of the unmodeled fluid flow. A video of experimental runs is shown at \href{https://youtu.be/9_wx0tdDta0}{https://youtu.be/9\textunderscore wx0tdDta0}.} \label{table:: exp_actuation}
	\label{fig:: OFFvsPID_Histogram} 
\end{figure*}

We implement the proposed data-driven Koopman methodology in two ways. One approach is computing the model and LQR gains offline once; the other is updating the model and recalculating the LQR gains online in real time. To update the Koopman operator online in a memory-efficient manner, we do not store any previous measurements. Instead, we use \eqref{eq:: Kd_AG} and split it into the $P$ measurements used last to calculate the Koopman operator and the $\Delta P$ new measurements since the last update, where $P_{total} = P + \Delta P$ is the total number of measurements used. Then, 
\begin{equation}
 \begin{aligned}
 \tilde{\mathcal{K}}_{d, new}^* =& \mathcal{A}_{new}\mathcal{G}_{new}^\dagger,
 \end{aligned}
\end{equation}
where
\begin{equation}
 \begin{aligned}
 \mathcal{A}_{new} =& \frac{1}{P_{total}} (\mathcal{A} P + \sum_{k = P}^{P_{total}-1} \Psi(s_{k+1}, u_{k+1}) \Psi(s_k, u_k)^T) \\
 \mathcal{G}_{new} =& \frac{1}{P_{total}} (\mathcal{G} P + \sum_{k = P}^{P_{total}-1} \Psi(s_{k}, u_{k}) \Psi(s_k, u_k)^T)
 \end{aligned}
\end{equation}
and $\mathcal{A}$ and $\mathcal{G}$ are given by \eqref{eq::AG}. The derivation of the formula is shown in Appendix \ref{App:: OnlineUpdateKoopman}. Then, the LQR gains are recomputed as shown in Section \ref{sec:: ControlSynthesis} using the updated Koopman operator. We show an outline of the process in Fig. \ref{fig:: methodOutline}.

The proposed method is compared to a PID controller, a widely-used model-free control method. PID feedback is quite effective, requires low computational effort, and does not require exact knowledge of the dynamics or any of the model coefficients \cite{PIDquevedo}. In fact, it has been shown to perform remarkably well on motion and speed control of robotic fish \cite{Control_Yu,ren2015motion}. However, this method often requires intensive gain tuning. To tune the PID controller, we utilize the Ziegler-Nichols' closed-loop method \cite{haugen2004pid}. By comparing our method to PID, we hope to demonstrate that our proposed method is a promising, data-driven feedback scheme that compares well against a popularly used, effective controller, yet without the additional overhead of intensive parameter tuning. The two methods are compared on tracking linear and circular trajectories that are described in terms of the desired orientation and forward velocity.

The comparison of the proposed Koopman-LQR method and the PID is conducted over ten trials for both types of trajectories. We further compare their performance in the presence of fluid disturbance, generated by a propeller (see Fig. \ref{fig:: propeller}). The results are presented in Fig. \ref{fig:: OFFvsPID_Histogram}, which displays the average error in the two tracked states, together with the standard error. The standard error is a measure of the expected variability in the average error, contrary to the standard deviation that measures the expected variability away from the mean. Note that, for the case of LQR with gains computed offline, the data used to obtain the Koopman representation was collected in the absence of fluid disturbance. 

From the results, the proposed data-driven Koopman-LQR scheme, regardless of whether it is updated online, outperforms PID in all tasks, with or without fluid disturbance, except for tracking the orientation in the linear trajectory, where both algorithms have similar performance. The difference in the performance between the Koopman-LQR and the PID is highlighted in the tracking of the circular trajectory in the presence of fluid flow, where the angle error is significantly higher for PID (Fig. \ref{fig:: OFFvsPID_Histogram}). Given that angle tracking (in the linear trajectory) is the only metric where PID is comparable to our proposed method, we attribute the difference in performance in the presence of fluid flow to the robustness of our approach and its ability to track the desired trajectory in a confined space in the presence of unmodeled dynamics. 

In the absence of fluid disturbance, the two implementations of the proposed method (that is, with and without updating the model and LQR gains in real time) have comparable results. It is only in tracking the circular trajectory that the offline implementation tracks the desired velocity better. We conjecture that this is due to the collisions of the robotic fish with the side wall (and the unmodeled boundary conditions) that take place because of the confined space. Introducing such discontinuous disturbances likely deteriorates the learned model temporarily; yet it still outperforms the well tuned PID controller. On the other hand, the major benefit of updating the model online is, as one would expect, in the presence of fluid disturbance. There, the real-time updated method significantly outperforms the offline-trained Koopman-LQR scheme, highlighting the importance of updating the model online in environments that constantly change.

\section{Discussion and Future Work} 
In this paper, we use the Koopman operator framework to develop derivative-based data-driven linear representations of nonlinear systems, suitable for real-time feedback. The proposed synthesis of the observable functions aims at minimizing the representation error without requiring knowledge of the dynamics. Utilizing Taylor series error bounds, we characterize the approximation error induced by the Koopman operator and use the error bound formula to decide the order of derivatives for synthesizing the basis functions. LQR is then used as an example of efficient control synthesis based on the trained Koopman operator. In fact, unless the model is updated online, the LQR gains are computed only once. We demonstrate the efficacy of our approach with simulation and experimental results using a case study of a tail-actuated robotic fish.

One of the most promising aspects of the derivative-based synthesis of Koopman operators is the application to completely unknown dynamics. In that case, the proposed method relies on numerically estimating the state derivatives, amplifying any noise present in the measurements and possibly worsening the accuracy of the error bound estimates. As simulation results suggest, however, after implementing a simple moving average filter, the performance of the Koopman operator remains robust even for high noise levels. While the error bound estimates may be violated, they remain close to the true bounds. Such results merit further research on more complicated dynamics using more sophisticated noise reduction schemes. In fact, as part of early-stage efforts of underwater exploration, we have tested our derivative-based Koopman approach to the unknown dynamics of a soft robotic fish \cite{castano2020}, where we numerically estimate state derivatives using high-gain observers and are able to predict with reasonable accuracy the evolution of the states.

The modeling performance of the proposed Koopman approach is comparable to state-of-the-art data-driven modeling methods, such as SINDy and NARX. This is in part shown via the comparison of MPC control performance, based on these different modeling schemes. On the other hand, as a linear model, the Koopman approach lends itself readily to efficient control design tools, an example of which is LQR. In fact, LQR based on a Koopman model delivers control performance comparable to these MPC-based methods. In short, this work develops a systematic approach for constructing the Koopman operator with high fidelity, which can be used for various model-based control schemes (as illustrated here with MPC and LQR). The advantage over other nonlinear data-driven methods is the linear nature of the model, which allows us to more efficiently compute the controller (e.g., LQR). This also justifies why we use LQR in the later simulation and experiments for the robotic fish example.

Although the proposed method can be used for any system that can benefit from data-driven methods or reduction of the nonlinearity, underwater robotics is perhaps the most suitable application for this method, due to the inherent environmental uncertainty, the highly nonlinear dynamics, and the need for controllers that use limited computation (to preserve battery use or due to limited computational power). While this method could certainly be applied to other systems, it perhaps would not be as useful for low-dimensional systems, with known dynamics and few nonlinearities. 

This work can be further expanded in multiple ways. We believe that the modeling is worsened by the average dynamics \eqref{eq:: dynamics} used to describe the longer-term behavior of the tail-actuated fish. In fact, the average model borrowed from \cite{averagemodel} is inaccurate, because it assumes that the tail bias has no effect on the angular velocity if the amplitude is zero ($u_2 = 0$ when $\alpha_a = 0$). However, when the system already has forward velocity, the orientation of the tail certainly induces a rotation to the movement. Thus, abandoning the average model, and instead using Kirchoff's equations for a rigid body in fluid environment, could improve the model fitness. Alternatively, one could also use a system identification algorithm, such as SINDy \cite{koopman_sindy}, to first obtain a model for the nonlinear dynamics of the system. Besides identifying the underlying dynamics, we are also interested in experimentally testing this methodology on entirely unknown dynamics where we use measurements to numerically compute the derivatives of the system states and populate the Koopman basis functions. Robotic fish with different morphologies, such as fish propelled with pectoral fish, are possible candidates for further testing. We are also interested in computing error bounds for general basis functions, not limited to derivatives of the dynamics. 

Last, note that work in \cite{haggerty2020modeling} uses Koopman operators with time delay observables to model and control (also using LQR feedback) a soft robotic arm. The authors comment that the time delay observables help better capture the momentum of the dynamics. We believe that time delay observables are equivalent to derivative-based observables given that the latter can be approximated numerically from past measurements \cite{derivatives_estimation}. Interestingly enough, the authors in \cite{haggerty2020modeling} observe a steady increase in predictive power as they increase the number of time delays in the observables, which we consider to be analogous to including higher-order derivatives in the basis functions. Investigating the connection between time delay observables and derivative-based observables can reveal whether the error bound formula presented in this work can be also utilized in other applications of Koopman models, without the need to numerically estimate derivatives \cite{arbabi2017computation, eivazi2020recurrent, kamb2020time, korda2020data}. 

\label{sec:conclusion}
\section*{Acknowledgment}
We would like to greatly thank the anonymous reviewers for their valuable feedback. 

\appendices
 \section{Global Error for Taylor-Based Koopman Operators} \label{App:: GlobalError}
In each time step, there is error induced in the updated function. Let $e^{(m)}_k$ indicate the deviation from the accurate value of the $f^{(m)}_k$ function at time $t_k$, where $m \in $ $\mathbb{Z}\,\cap\,[0,n]$. That is,
\begin{equation}
\begin{aligned} 
e_k =& \tilde{f}_k - f_k \\
e'_k =& \tilde{f}'_k - f'_k \\
\vdots& \\
e^{(n)}_k =& \tilde{f}^{(n)}_k - f^{(n)}_k
\end{aligned}
\end{equation}

Next, consider how previous errors accumulate in the prediction of a function: 
\begin{equation}
\begin{aligned}
\tilde{f}_{k+1} =& \tilde{f}_{k} + \tilde{f}'_{k} \cdot \Delta t + \tilde{f}''_{k} \cdot \frac{\Delta t^2}{{2!}} + \dots + \tilde{f}^{(n)}_{k} \frac{\Delta t^n}{n!}\\
=& (f_{k} + e_{k}) + (f'_{k} + e'_{k}) \cdot \Delta t + (f''_{k} + e''_{k}) \cdot \frac{\Delta t^2}{{2!}} \\&+ \dots + (f^{(n)}_{k} + e^{(n)}_{k}) \frac{\Delta t^n}{n!} \\
=& f_{k} + f'_{k} \cdot \Delta t + f''_{k} \frac{\Delta t^2}{{2!}} + \dots + f^{(n)}_{k} \frac{\Delta t^n}{n!} \\
&+ \underbrace{e_{k} + e'_{k} \cdot \Delta t + e''_{k} \cdot \frac{\Delta t^2}{{2!}} + \dots + e^{(n)}_{k} \frac{\Delta t^n}{n!}}_{error~terms},
\end{aligned}
\end{equation}
such that 
\begin{equation}
\begin{aligned}
e_{k+1} =& e_{k} + e'_{k} \cdot \Delta t + e''_{k} \cdot \frac{\Delta t^2}{{2!}} + \dots + e^{(n)}_{k} \frac{\Delta t^n}{n!} 
\\&+ \frac{\Delta t^{n+1}}{(n+1)!} f^{(n+1)}_{k, {k+1}},
\end{aligned}
\end{equation}
where the last error term is from Lagrange's remainder formula and is added at each step. Remember, $f^{(n)}_{k,k+1}$ is the $n$th derivative of a function evaluated at some $t \in [t_k, t_{k+1}]$. Similarly, 
\begin{equation}
\resizebox{0.99\hsize}{!}{$
	\begin{aligned}
e_{k} =& e_{k-1} + e'_{k-1} \cdot \Delta t + e''_{k-1} \cdot \frac{\Delta t^2}{{2!}} + \dots + e^{(n)}_{k-1} \frac{\Delta t^n}{n!} + \frac{\Delta t^{n+1}}{(n+1)!} f^{(n+1)}_{k-2, {k-1}} \\
\vdots& \\
e_3 =& e_{2} + e'_{2} \cdot \Delta t + e''_{2} \cdot \frac{\Delta t^2}{{2!}} + \dots + e^{(n)}_{2} \frac{\Delta t^n}{n!} + \frac{\Delta t^{n+1}}{(n+1)!} f^{(n+1)}_{2, {3}} \\
e_2 =& e_{1} + e'_{1} \cdot \Delta t + e''_{1} \cdot \frac{\Delta t^2}{{2!}} + \dots + e^{(n)}_{1} \frac{\Delta t^n}{n!} + \frac{\Delta t^{n+1}}{(n+1)!} f^{(n+1)}_{1, {2}}, 
\end{aligned}
$}
\end{equation}
where $e_1 = f^{(n+1)}_{0, 1}\frac{\Delta t^{n+1}}{(n+1)!}$; in the first iteration, we assume the knowledge of the exact values of all derivatives up to order $n$, such that the only numerical error comes from the inaccuracy of the Taylor series expansion. Without loss of generality, the functions are assumed to be known exactly at $k = 0$. That is, $e_0 = 0, e'_0 = 0, \dots, e^{(n)}_0 = 0$. Thus, we can express the error of a function $f$ as 
\begin{equation}
\begin{aligned}
e_k =& (\sum_{i = 1}^{k-1} e'_i \cdot \Delta t + e''_i \cdot \frac{\Delta t^2}{2!} + \dots + e_i^{(n)} \frac{\Delta t^n}{n!} + f_{i, i+1}^{(n+1)}\frac{\Delta t^{n+1}}{(n)!}) \\
&+ f^{(n+1)}_{0, 1}\frac{\Delta t^{n+1}}{(n+1)!}
\end{aligned}
\end{equation}
Similarly, for the errors associated with the higher-order terms, we have
\begin{equation}
\begin{aligned}
e'_k =&~ (\sum_{i = 1}^{k-1} e''_i \cdot \Delta t + e''_i \cdot \frac{\Delta t^2}{2!} + \dots + e_i^{(n)}\frac{\Delta t^{n-1}}{(n-1)!} \\& + f_{i, i+1}^{(n+1)}\frac{\Delta t^{n}}{(n)!}) + f^{(n+1)}_{0, 1}\frac{\Delta t^{n}}{(n)!}\\
\vdots& \\
e^{(n-1)}_k =&~(\sum_{i = 1}^{k-1} e^{(n)}_i \cdot \Delta t + f_{i, i+1}^{(n+1)}\frac{\Delta t^{(n+1)-(n-1)}}{((n+1)-(n-1))!}) 
\\&+ f^{(n+1)}_{0, 1}\frac{\Delta t^{(n+1)-(n-1)}}{((n+1)-(n-1))!}\\
e^{(n)}_k =&~(\sum_{i = 1}^{k-1} f_{i, i+1}^{(n+1)}\frac{\Delta t^{n+1-n}}{(n+1-n)!}) + f^{(n+1)}_{0, 1}\frac{\Delta t^{n+1-n}}{(n+1-n)!} 
\end{aligned}
\end{equation}
In general, the error of the $p$-th derivative of a function $f$ at time $t_k$ is given by 
\begin{equation}
\begin{aligned}
e^{(p)}_k =& \Bigg(\sum_{i = 1}^{k-1} \Big(\big(\sum_{j = 1}^{n-p} e_i^{(j)} \frac{\Delta t^j}{j!} \big) + f_{i, i+1}^{(n+1)}\frac{\Delta t^{n+1-p}}{(n+1-p)!}\Big)\Bigg) \\
&+ f^{(n+1)}_{0, 1}\frac{\Delta t^{n+1-p}}{(n+1-p)!},
\end{aligned}
\end{equation}
which can be simplified to 
\begin{align}\label{eq:: error_p0Derivative}
e^{(p)}_k =& \sum_{i = 1}^{k-1}\sum_{j = 1}^{n-p} e_i^{(j)} \frac{\Delta t^j}{j!} + \sum_{i = 0}^{k-1}f_{i, i+1}^{(n+1)}\frac{\Delta t^{n+1-p}}{(n+1-p)!},
\end{align}
which is split into the error from the derivative inaccuracies and the error induced by the Taylor series expansion at each step. Note that $n$ is the number of derivatives used to propagate $f$, where $p \in$ $\mathbb{Z}\,\cap\, [0, n]$ indicates the order of the derivative of a function for which the error is calculated ($p=0$ refers to the original function $f$).

\section{Global Error Bounds for Taylor-Based Koopman Operator}\label{App:: ErrorBounds}
First, consider the error in $f$ when no derivative basis functions are used, that is, $n=0$. Then,
\begin{align}\label{eq:: 0der}
e_k =& \sum_{i = 0}^{k-1} f_{i, i+1}^{(1)}\frac{\Delta t^{1}}{1!}\\
|e_k| =& |\sum_{i = 0}^{k-1} f_{i, i+1}^{(1)}\frac{\Delta t^{1}}{1!}| \\
|e_k| \le& \frac{\Delta t^{1}}{1!}\sum_{i = 0}^{k-1} |f_{i, i+1}^{(1)}| 
\end{align}
To further simplify the analysis, we can assume a maximum value of $f^{(1)}$, $|f_{i, i+1}^{(1)}| \le |f^{(1)}_{max}|$ for all $i \in$ $\mathbb{Z}\,\cap\,[0, k-1]$. Then, 
\begin{align}
|e_k| \le& \frac{\Delta t^{1}}{1!}\sum_{i = 0}^{k-1} |f^{(1)}_{max}| \\
|e_k| \le& k \cdot \Delta t\cdot |f^{(1)}_{max}| = T \cdot |f^{(1)}_{max}| \label{eq::ek_n0},
\end{align}
where $T \triangleq k \cdot \Delta t$ is the time window that we are approximating the function over. 

Similarly, we compute the error bound for $n = 1$ (one derivative of $f$ in the basis functions). Then, 
\begin{align}
e_k =& \sum_{i = 1}^{k-1} e'_i\Delta t + \sum_{i = 0}^{k-1} f_{i, i+1}^{(2)}\frac{\Delta t^{2}}{2!},
\end{align}
where $e'_i$ is given by \eqref{eq:: 0der}. Note that $f^{(1)}$ becomes $f^{(2)}$; what was $e_i$ is now $e_i'$ and the first-order derivative of $e'_i$ is the second-order derivative of $e_i$. Thus,
\begin{align}
e_k =& \sum_{i = 1}^{k-1} \big(\sum_{j = 0}^{i-1} f_{j, j+1}^{(2)}\frac{\Delta t^{1}}{1!}\big) \Delta t + \sum_{i = 0}^{k-1} f_{i, i+1}^{(2)}\frac{\Delta t^{2}}{2!} \\
|e_k| \le& \sum_{i = 1}^{k-1} \big(\sum_{j = 0}^{i-1} |f_{j, j+1}^{(2)}|\frac{\Delta t^{1}}{1!}\big) \Delta t + \frac{\Delta t^{2}}{2!}\sum_{i = 0}^{k-1}|f_{i, i+1}^{(2)}|
\end{align}
To further simplify things, we again use $|f_{i, i+1}^{(2)}| \le |f^{(2)}_{max}|$ for all $i \in$ $\mathbb{Z}\,\cap\,[0, k-1]$, such that
\begin{align}
|e_k| \le& \sum_{i = 1}^{k-1} \big(\sum_{j = 0}^{i-1} |f^{(2)}_{max}|\frac{\Delta t^{1}}{1!}\big) \Delta t + \frac{\Delta t^{2}}{2!}\sum_{i = 0}^{k-1} |f^{(2)}_{max}| \\
|e_k| \le& \sum_{i = 1}^{k-1} i|f^{(2)}_{max}|\Delta t^{2} + k\cdot \frac{\Delta t^{2}}{2} |f^{(2)}_{max}| \\
|e_k| \le& |f^{(2)}_{max}|\Delta t^{2} \sum_{i = 1}^{k-1} i+ k\cdot \frac{\Delta t^{2}}{2} |f^{(2)}_{max}|.
\end{align}
Using the property $\sum_{i =0}^n i= \sum_{i =1}^n i= \frac{n (n+1)}{2}$,
\begin{align}
|e_k| \le& |f^{(2)}_{max}|\Delta t^{2} \frac{(k-1)k}{2} + k\cdot \frac{\Delta t^{2}}{2} |f^{(2)}_{max}| \\
|e_k| \le& |f^{(2)}_{max}|\Delta t^{2} \frac{k^2 - k}{2} + k\cdot \frac{\Delta t^{2}}{2} |f^{(2)}_{max}| \\
|e_k| \le& \frac{(k\cdot \Delta t)^2}{2}|f^{(2)}_{max}| = \frac{T^2}{2}|f^{(2)}_{max}|.
\end{align}

From $n=0$ and $n=1$, we notice a pattern in the error bound expression. Using proof by induction, we next show that the error bound is given by 
\begin{align}
|e_k| \le& \frac{T^{n+1}}{(n+1)!}|f^{(n+1)}_{max}|.
\end{align}
\ssubsection{\textbf{Base Case:} $n = 0$}

From \eqref{eq::ek_n0}, it is true that, for $n = 0$, 
\begin{align}
|e_k| \le& T \cdot |f^{(1)}_{max}|.
\end{align}
\ssubsection{\textbf{Induction Step:}}
Assuming that the relationship holds for $n-1$, we show it is also true for $n$. For the case of using basis functions with derivatives up to order $n$, the error in the derivative function of order $p = 0$ is, using \eqref{eq:: error_p0Derivative}, given by 
\begin{align}
e_k =& \sum_{i = 1}^{k-1}\sum_{j = 1}^{n} e_i^{(j)} \frac{\Delta t^j}{j!} + \sum_{i = 0}^{k-1}f_{i, i+1}^{(n+1)}\frac{\Delta t^{n+1}}{(n+1)!}.
\end{align}
Taking the absolute value of the error,
\begin{align}
|e_k| =& |\sum_{i = 1}^{k-1}\sum_{j = 1}^{n} e_i^{(j)} \frac{\Delta t^j}{j!} + \sum_{i = 0}^{k-1}f_{i, i+1}^{(n+1)}\frac{\Delta t^{n+1}}{(n+1)!}| \\
 \le& \sum_{i = 1}^{k-1}\sum_{j = 1}^{n} |e_i^{(j)}| \frac{\Delta t^j}{j!} + \sum_{i = 0}^{k-1} |f_{i, i+1}^{(n+1)}|\frac{\Delta t^{n+1}}{(n+1)!},
\end{align}
where $\dfrac{\Delta t^j}{j!}$ and $\dfrac{\Delta t^{n+1}}{(n+1)!}$ are non-negative. Then, we use the relationship to substitute for the terms $|e_i^{(j)}|$. Note that $|e_i^{(1)}|$ is equivalent to $|e_k|$ using $n-1$ derivatives. Similarly, each term $|e_i^{(j)}|~\text{for}~j\in\mathbb{Z}\,\cap\,[1,n]$ is equivalent to $|e_k|$ for $n-j$ derivatives. Thus, we use the relationship that holds for up to $n-1$ derivatives, such that
\begin{equation}\begin{aligned}
|e_k| \le& \sum_{i = 1}^{k-1}\sum_{j = 1}^{n} |\frac{(i \Delta t)^{n+1-j}}{(n+1-j)!}f_{max}^{(n+1)}| \frac{\Delta t^j}{j!} \\
&+ \sum_{i = 0}^{k-1} |f_{i, i+1}^{(n+1)}|\frac{\Delta t^{n+1}}{(n+1)!} \\
=&|f_{max}^{(n+1)}| \Delta t^{n+1} \sum_{i = 1}^{k-1}\sum_{j = 1}^{n} \frac{i^{n+1-j}}{(n+1-j)!j!} \\
&+ |f_{max}^{(n+1)}|\sum_{i = 0}^{k-1} \frac{\Delta t^{n+1}}{(n+1)!},
\end{aligned}\end{equation}
where $|f_{max}^{(n+1)}| \ge |f_{i, i+1}^{(n+1)}|$. Next, we use $j' = n+1-j$ to simplify the inner sum term
\begin{align}
\sum_{j = 1}^{n} \frac{i^{n+1-j}}{(n+1-j)!j!} = \sum_{j' = n}^{1} \frac{i^{j'}}{(j')!(n+1-j')!}, 
\end{align}
which can be rewritten for $j$ and the summation order can also be reversed such that
\begin{equation}\begin{aligned}
|e_k| \le& |f_{max}^{(n+1)}| \Delta t^{n+1} \sum_{i = 1}^{k-1}\sum_{j = 1}^{n} \frac{i^{j}}{(n+1-j)!j!} \\
&+ |f_{max}^{(n+1)}|\sum_{i = 0}^{k-1} \frac{\Delta t^{n+1}}{(n+1)!}.
\end{aligned}\end{equation}
We can simplify 
\begin{align}
 |f_{max}^{(n+1)}|\sum_{i = 0}^{k-1} \frac{\Delta t^{n+1}}{(n+1)!} = k |f_{max}^{(n+1)}| \frac{\Delta t^{n+1}}{(n+1)!},
\end{align}
such that 
\begin{equation}\begin{aligned}
|e_k| \le& |f_{max}^{(n+1)}| \Delta t^{n+1} \sum_{i = 1}^{k-1}\sum_{j = 1}^{n} \frac{i^{j}}{(n+1-j)!j!} \\
&+ k |f_{max}^{(n+1)}| \frac{\Delta t^{n+1}}{(n+1)!} \\
=& \left[\sum_{i = 1}^{k-1}\sum_{j = 1}^{n} \frac{i^{j}}{(n+1-j)!j!} + \frac{k}{(n+1)!} \right] |f_{max}^{(n+1)}| \Delta t^{n+1}.
\end{aligned}\end{equation}
Using the binomial coefficient 
\begin{align}
\frac{a!}{(a-b)! b!} = \binom{a}{b}, 
\end{align}
to rewrite the term $(n+1 - j)! j!$, 
\begin{equation}
\resizebox{0.99\hsize}{!}{$
\begin{aligned}
|e_k| \le& \left[\sum_{i = 1}^{k-1}\sum_{j = 1}^{n} \frac{i^{j}}{(n+1)!} \binom{n+1}{j}+ \frac{k}{(n+1)!} \right] |f_{max}^{(n+1)}| \Delta t^{n+1}.
\end{aligned}
$}
\end{equation}
Switching the summation order, 
\begin{align}
|e_k| \le& \left[\sum_{j = 1}^{n}\binom{n+1}{j}\sum_{i = 1}^{k-1} i^{j} + k \right] |f_{max}^{(n+1)}| \frac{\Delta t^{n+1}}{(n+1)!}.
\end{align}
To use Pascal's identity \cite{PascalIdentity}:
\begin{align}
 \sum_{p=0}^{k}\binom{k+1}{p} \sum_{j=1}^{n} j^p = (n+1)^{k+1} - 1,
\end{align}
we rewrite the error bound as
\begin{equation}\begin{aligned}
|e_k| \le& \left[\sum_{j = 0}^{n}\binom{n+1}{j}\sum_{i = 1}^{k-1} i^{j} - \binom{n+1}{0}\sum_{i = 1}^{k-1} i^{0}+ k \right] \\
&|f_{max}^{(n+1)}| \frac{\Delta t^{n+1}}{(n+1)!},
\end{aligned}\end{equation}
such that
\begin{equation}\begin{aligned}
|e_k| \le& \left[(k^{n+1} - 1 ) - (k-1) + k \right]|f_{max}^{(n+1)}| \frac{\Delta t^{n+1}}{(n+1)!} \\
=& k^{n+1}|f_{max}^{(n+1)}| \frac{\Delta t^{n+1}}{(n+1)!}\\
=& \frac{T^{n+1}}{(n+1)!}|f_{max}^{(n+1)}|,
\end{aligned}\end{equation}
where $T = k \cdot \Delta t$. Therefore, 
\begin{align}
|e_k| \le& \frac{T^{n+1}}{(n+1)!}|f^{(n+1)}_{max}|~\text{for all}~n \,\in \mathbb{Z}^{\ge 0}.
\end{align}

\subsection{Incremental Update of Koopman Operator}\label{App:: OnlineUpdateKoopman}
Consider $P$ measurements used to last update the Koopman operator and $\Delta P$ new measurements. We incrementally compute a Koopman operator using all $P_{total} = P + \Delta P$ measurements as follows. The Koopman operator is computed using \eqref{eq:: Kd_AG}, which we split the expression into past and new measurements, such that
\begin{equation}
 \begin{aligned}
 \mathcal{A}_{new} =& \frac{1}{P_{total}} \sum_{k = 0}^{P_{total}-1} \Psi(s_{k+1}, u_{k+1}) \Psi(s_k, u_k)^T \\
 =& \frac{1}{P_{total}} (\sum_{k = 0}^{P-1} \Psi(s_{k+1}, u_{k+1}) \Psi(s_k, u_k)^T \\
 &+ \sum_{k = P}^{P_{total}-1} \Psi(s_{k+1}, u_{k+1}) \Psi(s_k, u_k)^T)
 \end{aligned}
\end{equation}
and, using 
\begin{equation}
 \begin{aligned}
 \mathcal{A} = \frac{1}{P} \sum_{k = 0}^{P-1} \Psi(s_{k+1}, u_{k+1}) \Psi(s_k, u_k)^T,
 \end{aligned}
\end{equation}
we can rewrite $\mathcal{A}_{new}$ as 
\begin{equation}
 \begin{aligned}
 \mathcal{A}_{new} =& \frac{1}{P_{total}} (P \mathcal{A} + \sum_{k = P}^{P_{total}-1} \Psi(s_{k+1}, u_{k+1}) \Psi(s_k, u_k)^T.
 \end{aligned}
\end{equation}
Similarly, 
\begin{equation}
 \mathcal{G}_{new} =\frac{1}{P_{total}} (\mathcal{G} P + \sum_{k = P}^{P_{total}-1} \Psi(s_{k}, u_{k}) \Psi(s_k, u_k)^T).
\end{equation}

\ifCLASSOPTIONcaptionsoff
  \newpage
\fi


\bibliographystyle{IEEEtran}
\bibliography{references}

%

\vfill\eject
\begin{IEEEbiography}[{\includegraphics[width=1in,height=1.25in,clip]{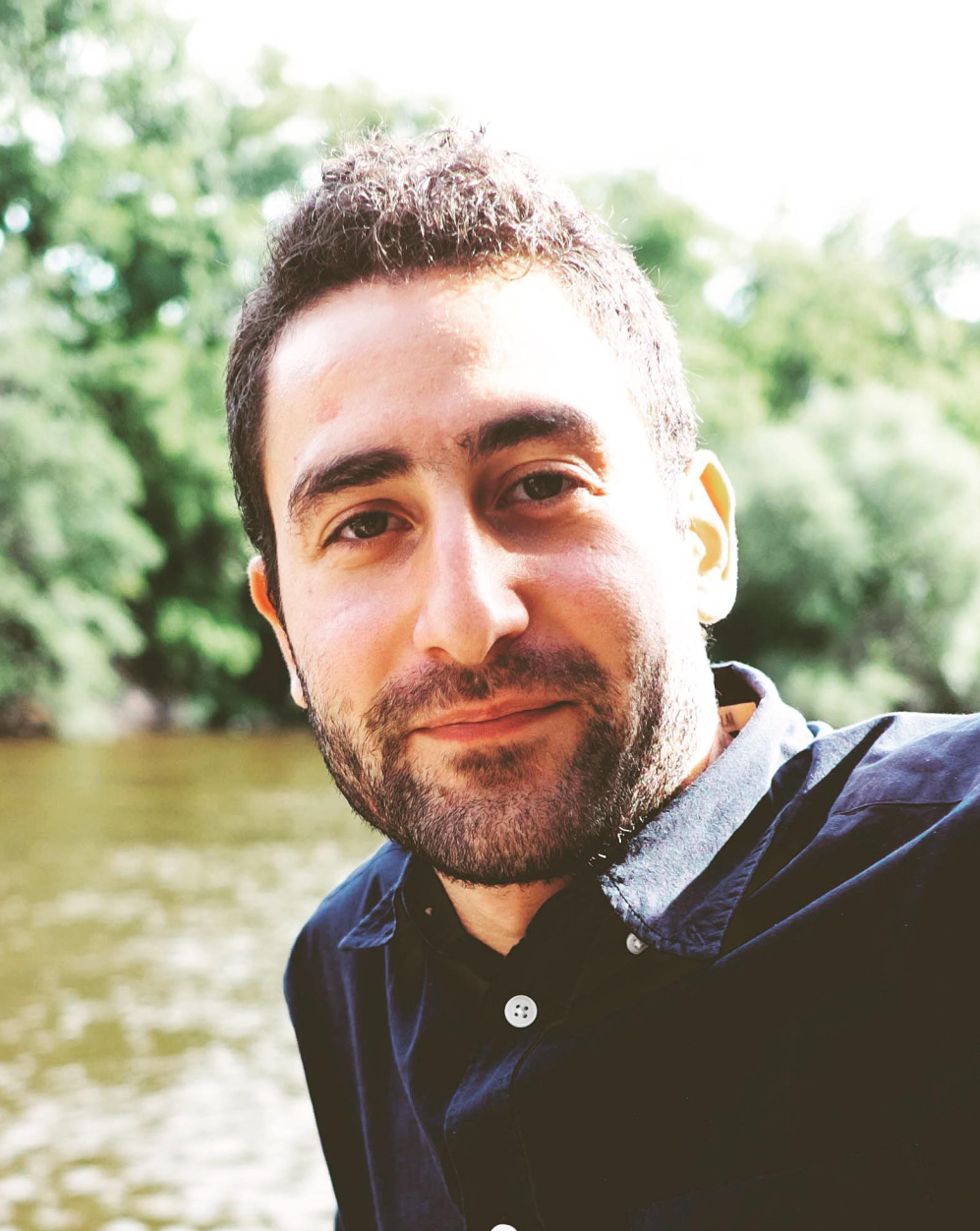}}]{Giorgos Mamakoukas}
received the B.A. degree in Physics at Grinnell College, IA, USA in 2014 and the M.S. in Mechanical Engineering at Northwestern University, Evanston, IL, in 2017. He is currently pursuing the Ph.D. degree in Mechanical Engineering at Northwestern University under the supervision of Todd Murphey. His interests include control, robotics, optimization, and machine learning.\end{IEEEbiography}
\begin{IEEEbiography}[{\includegraphics[width=1in,height=1.25in,clip,keepaspectratio]{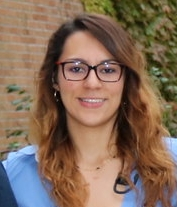}}]{Maria L. Casta\~{n}o}
received her B.S. in Electrical Engineering from Florida International University, Miami, FL, USA in May 2014. She is currently pursuing her  Ph.D degree at Michigan State University, East Lansing, MI, USA under the supervision of Xiaobo Tan. Her research interest include modeling, controls, and underwater robotics.\end{IEEEbiography}
\begin{IEEEbiography}[{\includegraphics[width=1in,height=1.25in,clip,keepaspectratio]{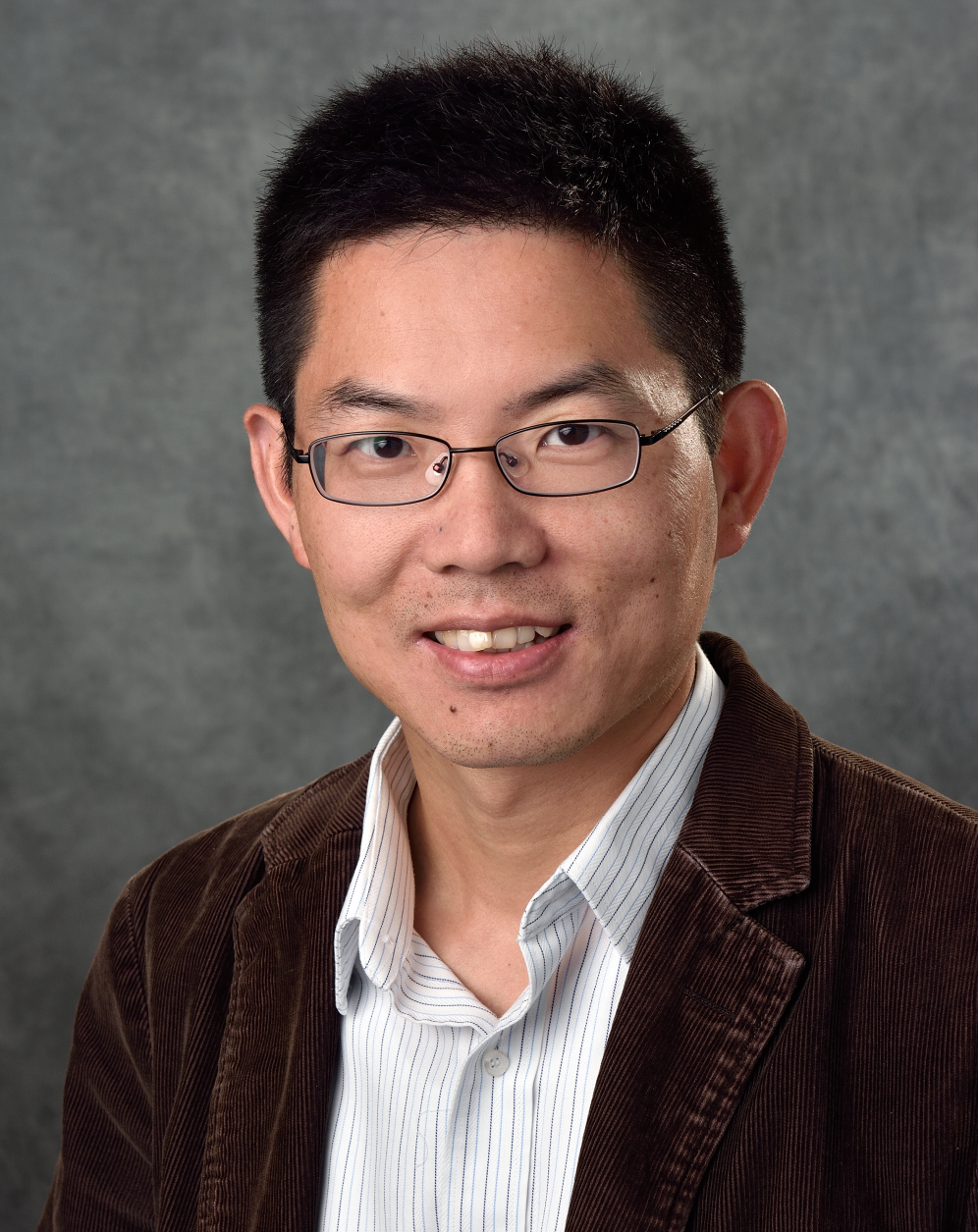}}]{Xiaobo Tan}
received the B.Eng. and M.Eng. degrees in automatic control from Tsinghua University, Beijing, China, in 1995 and 1998, respectively, and the Ph.D. degree in electrical engineering from the University of Maryland, College Park, in 2002. He is currently an MSU Foundation Professor and the Richard M. Hong Endowed Chair in the Department of Electrical and Computer Engineering (ECE) at Michigan State University (MSU). His research interests include control systems, smart materials, underwater robotics, and soft robotics.\\ 
Dr. Tan is a Senior Editor for IEEE/ASME Transactions on Mechatronics. He has coauthored over 250 peer-reviewed journal and conference papers, and holds four US patents. He is a Fellow of IEEE and ASME, and a recipient of NSF CAREER Award (2006), MSU Teacher-Scholar Award (2010), MSU College of Engineering Withrow Distinguished Scholar Award (2018), Distinguished Alumni Award from the ECE Department at University of Maryland (2018), and multiple best paper awards.\end{IEEEbiography}
\begin{IEEEbiography}[{\includegraphics[width=1in,height=1.25in,clip,keepaspectratio]{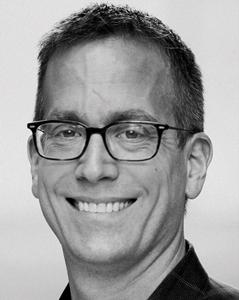}}]{Todd Murphey}
received his B.S. degree in mathematics from the University of Arizona and the Ph.D. degree in Control and Dynamical Systems from the California Institute of Technology.
He is a Professor of Mechanical Engineering at Northwestern University. His laboratory is part of the Center for Robotics and Biosystems, and his research interests include robotics, control, computational methods for biomechanical systems, and computational neuroscience.
Honors include the National Science Foundation CAREER award in 2006, membership in the 2014-2015 DARPA/IDA Defense Science Study Group, and Northwestern's Professorship of Teaching
Excellence. He was a Senior Editor of the IEEE Transactions on Robotics.
\end{IEEEbiography}
\vfill

\end{document}